\documentclass{article}
\usepackage{colt11e}
\usepackage{amsmath}
\usepackage{amssymb}
\usepackage{hyperref}

\usepackage[varg]{txfonts} % nicer times font (pxfonts for Palatino)
\usepackage[round,comma]{natbib}
\newcommand{\ignore}[1]{}
\newenvironment{prevproof}[2]{\noindent {\bf {Proof of {#1}~\ref{#2}:}}}{\hfill{$\blacksquare$}\vskip \belowdisplayskip}

\newtheorem{observation}{Observation}[section]

\newcommand{\bbP}{\mathbb{P}}
\newcommand{\bbQ}{\mathbb{Q}}

\def\Y{\mathcal{Y}}
\def\X{\mathcal{X}}
\def\L{\mathcal{L}}

\newcommand{\ind}{\mathbf{1}}

\newcommand{\dtv}{d_{\mathrm TV}}
\newcommand{\dk}{d_{\mathrm K}}

\newcommand{\R}{{\bf R}}

\newcommand{\E}{{\bf E}}
\newcommand{\EE}[1]{{\bf E}\left[ #1\right]}

\newcommand{\Var}{\mathrm{Var}}

\newcommand{\eps}{\epsilon}
\newcommand{\littlesum}{\mathop{{\textstyle \sum}}}
\newcommand{\littleprod}{\mathop{{\textstyle \prod}}}

\newcommand{\poly}{\mathrm{poly}}

\newcommand{\polylog}{\mathrm{polylog}}

\newcommand{\sob}{{\mathcal{S}}}

   \newcommand{\inote}[1]{\footnote{{\bf [[Ilias: {#1}\bf ]] }}}
   \newcommand{\remove}[1]{\par $<<${\it removed part}$>>$}
   \newcommand{\old}[1]{}

\title{Learning transformed product distributions}

\author{Constantinos Daskalakis\thanks{Research supported by NSF CAREER award CCF-0953960 and by a
Sloan Foundation Fellowship.}\\
MIT\\
{\tt costis@csail.mit.edu}
\And
Ilias Diakonikolas\thanks{Research supported by a Simons Foundation Postdoctoral Fellowship. Most of this work was done while at Columbia University, supported by NSF grant CCF-0728736, and by an Alexander S. Onassis Foundation Fellowship.}\\
UC Berkeley\\
{\tt ilias@cs.berkeley.edu}\\
\And
Rocco A. Servedio
\thanks{Supported by NSF grants CCF-0347282, CCF-0523664 and CNS-0716245, and
by DARPA award HR0011-08-1-0069.}\\
Columbia University\\
{\tt rocco@cs.columbia.edu}}

\begin{document}

\setcounter{page}{0}

\maketitle

\thispagestyle{empty}
\begin{abstract}

%We introduce a new type of learning problem:  learning an unknown collection
%of independent Bernoulli random variables $X_1,\dots,X_n$ (i.e.
%an unknown \emph{product distribution}) using samples $f(X)$ where $f$
%is a known transformation function.  Each choice of a transformation function
%$f$ specifies a learning problem in this framework.

We consider the problem of learning an unknown product distribution
$X$ over $\{0,1\}^n$ using samples $f(X)$ where $f$ is a
\emph{known} transformation function.  Each choice of a transformation function
$f$ specifies a learning problem in this framework.

Information-theoretic arguments show that for every transformation function
$f$ the corresponding learning problem can be solved to accuracy $\eps$,
using $\tilde{O}(n/\eps^2)$ examples, by a generic algorithm whose running
time may be exponential in $n.$
We show that this learning problem can be computationally intractable
even for constant $\eps$ and rather simple transformation functions.
Moreover, the above sample complexity bound is nearly optimal for the general
problem, as we
give a simple explicit linear transformation function $f(x)=w
\cdot x$ with integer weights $w_i \leq n$  and prove that the
corresponding learning problem requires $\Omega(n)$ samples.
%Added by Ilias, feel free to rephrase.

% I rephrased this paragraph to make it clear what is our main result.

% Rocco's previous version was:

%As our main positive result we give two highly efficient algorithms
%for learning a sum of independent unknown Bernoulli random variables,
%corresponding to the transformation function $f(x)= \sum_{i=1}^n x_i$.
%The first algorithm learns to
%$\eps$-accuracy in poly$(n)$ time, using a surprising poly$(1/\eps)$
%number of samples that is independent of $n.$  The second
%algorithm uses $\log(n) \cdot \poly(1/\eps)$ samples but has running time that is only
%$\poly(\log n, 1/\eps).$

As our main positive result we give a highly efficient algorithm
for learning a sum of independent unknown Bernoulli random variables,
corresponding to the transformation function $f(x)= \sum_{i=1}^n x_i$.
Our algorithm learns to
$\eps$-accuracy in poly$(n)$ time, using a surprising poly$(1/\eps)$
number of samples that is independent of $n.$
We also give an efficient algorithm that uses $\log n \cdot \poly(1/\eps)$
samples but has running time that is only $\poly(\log n, 1/\eps).$

\ignore{
In contrast, we also
show that for a simple explicit weighted linear transformation function $f(x)=w
\cdot x$ with integer weights $w_i \leq n,$ the learning problem requires
$\Omega(n)$ samples.
Finally, we
also give efficient learning algorithms for certain restricted types of
linear transformation functions $f(x)=w \cdot x$,
}

\end{abstract}

%\newpage

\section{Introduction}

We consider the problem of learning an unknown product
distribution that has been transformed according to a known function $f$.
This is a simple and natural learning problem, but one which does not seem
to have been explicitly studied in a systematic way from a computational
learning theory perspective.

More precisely, in this paper we restrict our model to the natural case
when the input distribution is a a product distribution over the
Boolean cube $\{0,1\}^n$. In
this learning scenario the learner is provided with samples from the
random variable $f(X)$, where $X=(X_1,\dots,X_n)$ is a vector of
independent 0/1 Bernoulli random variables $X_i$ whose expectations are
unknown to the learner.  We write $\overline{p}=(p_1,\dots,p_n) \in
[0,1]^n$ to denote $\E[X]$, and refer to $\overline{p}$ as the
\emph{target vector of probabilities}; we shall sometimes write
$f(\overline{p})$ to denote the random variable $f(X)$ described above.
Using these samples, the learner must with probability $1 - \delta$
\footnote{For ease of exposition we state all our positive results
throughout the paper with $\delta$ fixed to 1/10.  All our results extend
to general $\delta$ with a $\log(1/\delta)$ overhead in sample
complexity.} output a hypothesis distribution
% Ilias: I rehprased this a bit; Rocco's phrasing referred to the proper learning problem here too.
$\mathcal{H}$ over $f(\{0,1\}^n)$ such that
the total variation distance $\dtv(f(X),\mathcal{H})$ is at most $\eps$.
A \emph{proper} learning algorithm in this framework outputs
a hypothesis vector $\hat{p} \in [0,1]^n$ defining a hypothesis distribution $f(\hat{X})$, where
$\hat{X} = (\hat{X}_1,\dots,\hat{X}_n)$ is a vector of independent
0/1 Bernoulli random variables $\hat{X}_i$ whose expectation is
$\E[\hat{X}]=\hat{p}.$

%%%%%%%%%%%% Uncomment the two-paragraphs below to roll-back to previous version
%
%
%This paper initiates the study of a new type of unsupervised learning problem:  the
%problem of learning an unknown product distribution over
%$\{0,1\}^n$ that has been transformed according to a known function $f.$  This is a simple and natural learning problem, but one which does not seem to have been studied previously in a systematic way.
%
%In this learning scenario the learner is provided with samples from the random variable $f(X)$,
%where $X=(X_1,\dots,X_n)$ is a vector of independent 0/1 Bernoulli random variables $X_i$ whose
%expectations are unknown to the learner.  We write $\overline{p}=(p_1,\dots,p_n) \in [0,1]^n$ to denote
%$\E[X]$, and refer to $\overline{p}$ as the \emph{target vector of probabilities}; we shall
%sometimes write $f(\overline{p})$ to denote the random variable $f(X)$ described above.  Using these
%samples, the learner must with probability $1 - \delta$ \footnote{For ease of exposition we state all our positive results throughout the paper with $\delta$ fixed to 1/10.  All our results extend to general $\delta$ with a $\log(1/\delta)$ overhead in sample complexity.}
%output a hypothesis vector $\hat{p} \in [0,1]^n$ such that the total variation distance
%$\dtv(f(X),f(\hat{X}))$ is at most $\eps$, where $\hat{X} = (\hat{X}_1,\dots,\hat{X}_n)$ is the vector of independent 0/1 Bernoulli random variables $\hat{X}_i$ whose expectation is $\E[\hat{X}]=\hat{p}.$
%
%

We emphasize that in this learning scenario, the transformation function
$f$ is \emph{fixed and known to the learner}; the choice of a particular
transformation function $f$ specifies a particular learning problem in
this model, much as the choice of a concept class ${\cal C}$ specifies a
learning problem in Valiant's PAC learning model.  We will be interested
in both the computational complexity (running time) and sample complexity
(number of samples required) for algorithms that solve this problem, for
different transformation functions $f$.

\subsection{Motivation, examples, and connection to prior work}
\label{sec:priorexamples}

Our motivation for considering this model is twofold.  First, we feel that
it is so simple and natural as to warrant study for its own sake.
Second, we believe that it offers a useful perspective on
modeling probability distributions in settings where the
underlying source of randomness is not directly accessible to the learner.
In many settings we may wish to understand some phenomenon
(in the physical world, in a market, etc.) where the available observations
can be viewed as the output
of a transformation $f$ applied to some underlying random source $X$;
learning an accurate approximation of the distribution of $f(X)$ is a
natural goal in such a setting.  (The restriction on $X$ imposed in this
paper -- that it is a product of independent Bernoulli random variables --
admittedly represents an idealized scenario, but it is a natural starting
point for theoretical study.)  It is plausible that in such a situation
the transformation function $f$ may be well understood (as a consequence
of our knowledge of the laws governing the physical world, the
marketplace, etc.), but that much less is known about the parameters of
the underlying random variable $X$ (we may have no direct access to this
random variable, it could represent private information, etc.).  This
corresponds to our model's assumption that $f$ is ``known'' and the task
is to infer the parameters of $X$ that give rise to the observed data.

\smallskip

{\bf Examples:}  As a simple example to illustrate our learning model, we
consider the product distribution learning problem for $f$ where $f$ is
any read-once AND-gate function.  A function $f: \{0,1\}^n \to \{0,1\}^m,$ $f(x) =
(f_1(x),\dots,f_{m}(x))$ is a \emph{read-once AND-gate function} if each
$f_i(x)$ is an AND over some subset $S_i \subseteq [n]$ of the $n$ input
bits $x_1,\dots,x_n$ and the sets $S_1,\dots,S_m$ are pairwise disjoint.
It is not hard to see that
there is a straightforward proper learning algorithm based on linear
programming that succeeds for any read-once AND-gate function regardless of the
fanin of the AND gates (see Appendix~\ref{ap:andgatefunction}):

\begin{observation} \label{obs:andgate}

Let $f: \{0,1\}^n \to \{0,1\}^{m}$ be any fixed read-once AND-gate function
(known to the learner).  There is an
algorithm that uses poly$(n,1/\eps)$ samples from the target distribution
$f(\overline{p})$, runs in poly$(n,1/\eps)$ time, and with probability at
least 9/10 outputs a hypothesis vector $\hat{p}$ such that
$\dtv(f(\overline{p}),f(\hat{p}))\leq \eps.$

\end{observation}

%\ignore{
%between and has a different ``flavor'' than
%unsupervised learning problems that have been addressed to date in
%computational learning theory.  In most unsupervised learning problems
%(such as learning an unknown mixture of Gaussians \cite{}, learning an unknown mixture of product %distributions \cite{}, learning an unknown linear transformation, learning a probability
%distribution corresponding to an evolutionary tree \cite{}, etc.) the difficulty comes from inferring
%the unknown transformation
%}

\ignore{
Kearns \emph{et al} \cite{KMR+:94} studied the learnability of the class
of all AND-gate circuits in their model (their results are stated for
OR-gate distributions, but the two are easily seen to be equivalent by
Boolean duality).  They gave an algorithm that learns an unknown AND-gate
distribution to perfect accuracy (i.e. the error $\eps$ is 0) in their
model in time poly$(n,k^k)$ if the fan-in $|S_i|$ of each AND gate in the
target is at most $k$. Their algorithm works by iteratively constructing
the correct AND gate for each successive output bit of the unknown
circuit.
}

As a second example, we point out that the transformed product
distribution learning model is broad enough to encompass
the problem of learning an unknown
\emph{mixture of $k$ product distributions} over $\{0,1\}^n$ that
was considered by \cite{FreundMansour:99, CGG:02, FOS:08}.
%
%We now describe how the problem of learning a mixture of $k$
%product distributions can be viewed as a special case of
%the transformed product distribution learning model.
For simplicity we describe the
case $k=2$:  there are unknown product distributions
$\overline{p},\overline{q}$ over $\{0,1\}^n$ and unknown mixing weights
$\pi_p, \pi_q=1-\pi_p$.  The learner is given independent draws from the
mixture distribution (each draw is independently taken from $\overline{p}$
with probability $\pi_p$ and from $\overline{q}$ with probability
$\pi_q$), and must output hypothesis product distributions
$\hat{p},\hat{q}$ and hypothesis mixing weights $\hat{\pi}_p$,
$\hat{\pi}_q$.
This problem is easily seen to be equivalent to the transformed product
distribution learning problem for the function
$f: \{0,1\}^{2n+1} \to \{0,1\}^n$ which is such that on input
$(z,x_1,\dots,x_n,y_1,\dots,y_n) \in \{0,1\}^{2n+1}$ the $i$-th bit of
$f$'s output is $z x_i + (1 - z)y_i$.  It is easy to see that if the
target vector of probabilities for $f$ is $(\pi_p, p_1,\dots,p_n, q_1,
\dots q_n)$ then samples of $f$ are distributed exactly according to the
mixture of product distributions, and finding a good hypothesis vector in
$[0,1]^{2n+1}$ amounts to finding a hypothesis mixing weight $\hat{\pi}_p$
and hypothesis product distributions $\hat{p}, \hat{q}$ as required in the
original ``learning mixtures of product distributions'' problem.

\smallskip

{\bf Connection to prior work:}
The transformed product distribution learning model is related to
the PAC-style model of learning discrete probability
distributions that was introduced by \ignore{Kearns \emph{et al} }\cite{KMR+:94}
and studied in several subsequent works of \cite{Naor:96, ADFK:97,
FarachKannan:99, FreundMansour:99, CGG:02, FOS:08}.  In the \cite{KMR+:94}
framework a learning problem is defined by a class ${\cal C}$ of Boolean
circuits, and an instance of the problem corresponds to the choice of a
specific (unknown to the learner) target circuit $C \in {\cal C}$.  The
learner is given samples from $C(X)$ where $X$ is a uniform random string
from $\{0,1\}^m$, and the learner must with high probability output  a
hypothesis circuit $C'$ such that the random variable $C'(X)$ is
$\eps$-close to $C(X)$ (in KL-divergence).

Strictly speaking the transformed product distribution learning model
may be viewed as a special case of the Kearns \emph{et al} model.
This is done by considering a circuit class ${\cal C}$ that has
a circuit $C=C_{\overline{p}}$ for each possible
product distribution $\overline{p}$ over $\{0,1\}^n$; the circuit
$C_{\overline{p}}$
first transforms the uniform distribution over $\{0,1\}^m$ to the
product distribution $\overline{p}$ over $\{0,1\}^n$ and then applies
the transformation function $f$ to the output of $\overline{p}.$
However, learning problems ${\cal C}$ of this sort do not seem to have been
previously considered
in the \cite{KMR+:94} model, and we feel it is more natural to view
our model as dual in spirit to the earlier model.
In \cite{KMR+:94} the learner's task is to infer an \emph{unknown}
transformation (the circuit
$C$) into which are fed $n$-bit strings that are \emph{known} to be
distributed uniformly. In our case the transformation function $f$ is
\emph{known} to the
learner but the underlying product distribution that is fed into $f$ is
\emph{unknown} and must be inferred.

\ignore{

The above two examples both deal with the product distribution learning problem where $f$ is a multi-output Boolean-valued function.  While many natural and interesting questions about the model suggest themselves for such multi-output functions, we leave these
for future work; for most of this paper we focus on the case in which $f(X)$ is simply a real-valued random variable,  i.e. $f$ is a
transformation mapping $\{0,1\}^n$ to $\R.$

}

\subsection{Our results and techniques}

We establish a range of positive and negative results for this learning
problem, both for general functions and for particular
transformation functions of interest. For most of this paper we focus on the
case in which $f(X)$ is simply a real-valued random variable,  i.e. $f$ is a
transformation mapping $\{0,1\}^n$ to $\R.$

\ignore{ We feel that our model introduces fertile ground for future work
and that many interesting results can be obtained in the transformed
product distribution learning model beyond the results in this paper;
specific suggested research directions are given in
Section~\ref{sec:conclusion}. }

We begin by considering the most general possible setting, in which
the transformation function $f$ can be any
function mapping the domain $\{0,1\}^n$ into any range.
By an approach similar to the algorithm of
\cite{DL:01} for choosing a density estimate, we show (Theorem~\ref{thm:log-cover-size}) that
if the space $\{f(\overline{p})\}_{\overline{p} \in [0,1]^n}$ of
all $f$-transformed product distributions has an $\eps$-cover of size $N$,
then there is a generic learning algorithm for the $f$-transformed product
distribution problem that uses $O((\log N)/\eps^2)$ samples.
The algorithm works by carrying out a tournament that
matches every pair of distributions in the cover against each
other; our analysis shows that with high probability some
$\eps$-accurate distribution in the cover will survive the tournament
undefeated, and that any undefeated tournament will with high probability
be highly accurate.

\ignore{
A natural
first attempt to prove a result of this sort is to use a maximum
likelihood approach such as Theorem~16 of \cite{FOS:08}.  Such an approach
can indeed be used to prove a result of the desired type, but the sample
complexity resulting from this straightforward approach is poorer than we
would like, roughly $O((\log N)^3/\eps^4)$ (intuitively this is because
the maximum likelihood approach essentially works with KL-divergence, and
some loss is incurred in going to total variation distance). To avoid this
loss and obtain the stronger $O((\log N)/\eps^2)$ bound of
Theorem~\ref{thm:log-cover-size} (which is nearly optimal as discussed
below), we take a different approach.  Instead of a maximum likelhood
algorithm,}

As an immediate
consequence of the general result Theorem~\ref{thm:log-cover-size} we
get that for any transformation function $f$ there is an algorithm that
learns to accuracy $\eps$ using $\tilde{O}(n/\eps^2)$ samples:

\begin{theorem} [Information-theoretic upper bound for any $f$]
\label{thm:upper} Let $f: \{0,1\}^n \to \Omega$ be an arbitrary function where
$\Omega$ is any range set.  There is an algorithm
that uses $O((n/\eps^2) \cdot \log(n/\eps))$ samples from the target
distribution $f(\overline{p})$, runs in time $(n/\eps)^{O(n)}$, and with probability at least $9/10$
outputs a hypothesis vector $\hat{p}$ such that
$\dtv(f(\overline{p}),f(\hat{p})) \leq \eps.$ \end{theorem}

Since an $\eps$-cover of the space of
all $f$-transformed product distributions may have size
exponential in $n$, Theorem~\ref{thm:upper} does not
in general provide a computationally efficient algorithm.
Indeed, in Appendix~\ref{ap:nphard} we show that the learning
problem can be computationally hard even for rather simple
transformation functions:  using a reduction to the
PARTITION problem, we prove:

\begin{theorem} [NP-hardness] \label{thm:nphard}  Suppose NP $\not \subseteq$ BPP.  Then
there is an explicit degree-2 polynomial $f: \{0,1\}^n \to \R$ such that there is no
polynomial-time algorithm that solves the transformed product distribution
learning problem for $f$ to accuracy $\eps = 1/3.$ \end{theorem}

We also show that even for a simple linear transformation function $f(x) = w \cdot x$ with small
integer weights, it can be impossible to significantly improve on the $\tilde{O}(n)$
sample complexity of the generic algorithm from Theorem~\ref{thm:upper}.  In Appendix~\ref{sec:infolower}
we prove:

\begin{theorem} [Sample complexity lower bound] \label{thm:linearlower}
Fix any even $k \leq n$ and
let $f(x) = \sum_{i=k/2+1}^k i x_i.$ Let $L$ be any learning algorithm
that outputs a hypothesis vector $\hat{p}$ such that
$\dtv(f(\overline{p}),f(\hat{p})) \leq 1/40$ with probability at least
$e^{-o(k)}.$ Then $L$ must use $\Omega(k)$ samples from $f(\overline{p}).$
\end{theorem}

These negative results provide strong motivation for considering what is perhaps the most natural of all
transformation functions mapping $\{0,1\}^n$ to $\R$, the sum
$f(x) = \sum_{i=1}^n x_i$; we
refer to the corresponding learning problem as ``learning an unknown sum
of Bernoulli random variables.''  Our main contribution is a detailed study of this
learning problem.

\medskip

\noindent {\bf Learning sums of Bernoullis from constantly many samples.}  As our
main result, we show that any sum of
independent unknown Bernoulli random variables can be efficiently approximated to
$\eps$-accuracy by a proper learning algorithm that uses $\poly(1/\eps)$
samples, independent of $n.$  More precisely, we prove:

\begin{theorem} [Learning sums of Bernoullis from constantly many samples] \label{thm:sob}
 Let $f(x) = \sum_{i=1}^n x_i$.  There is an algorithm that uses
poly$(1/\eps)$ samples from the target distribution $f(\overline{p})$,
runs in time $n^3 \cdot \poly(1/\eps) + n \cdot (1/\eps)^{O(\log^2(1/\eps))}$, and
with probability at least $9/10$ outputs a hypothesis vector $\hat{p} \in
[0,1]^n$ which is such that $\dtv(f(\overline{p}),f(\hat{p})) \leq \eps.$
\end{theorem}

\noindent
It should be stressed that the generic algorithm of Theorem~\ref{thm:log-cover-size} requires
$\Omega((1/\eps^2) \cdot \log n)$ samples for this learning problem (easy
arguments give an $n^{\Omega(1)}$ lower bound
on the size of any cover for sums of Bernoullis). We view this sample complexity independent of $n$ as a surprising result
which may find subsequent applications. We next give a brief overview of the obstacles that appear and the
techniques involved in our proof.

% Give one sentence about tools, before we elaborate?

As a first step in the proof of Theorem~\ref{thm:sob}, we observe that a simple learning
algorithm using $O(1/\eps^2)$ samples gives a hypothesis
which has error at most $\eps$ with respect to the \emph{Kolmogorov
distance} (see Section~\ref{sec:prelim}). While the algorithm itself is
simple, its analysis relies on a fundamental result from probability
theory, known as the \emph{ Dvoretzky-Kiefer-Wolfowitz inequality}
(\cite{DKW56}), which may be viewed as a special case of the fundamental
Vapnik-Chervonenkis theorem (see Chapter~3 of \cite{DL:01}).  In Appendix~\ref{ap:DKW}
we give a self-contained proof of the DKW inequality using elementary techniques
(martingales and the method of bounded differences) and an interesting trick that goes back to
Kolmogorov (see \cite{Peres:Soda09}); this proof is significantly different from the proofs we know of
in the probability literature (see \cite{DKW56,Massart90}, and
Chapter~3 of \cite{DL:01}).

A natural attempt is to use Kolmogorov approximation as a black box
to obtain an approximation in total variation distance.  In the special
case when $X$ is a sum of Bernoulli random variables and $Y$ is a binomial
distribution $B(n,p)$ (i.e. all the Bernoullis $Y_1,\dots,Y_n$ have the
same mean $p$), it is indeed possible to bound the total variation
distance between $X$ and $Y$ as a function of their Kolmogorov distance.
It can be shown that in this case the distributions of $X$ and $Y$ cross
each other at most a constant number of times, and this is easily seen to
imply that the two distances (total variation and Kolmogorov) are within a
constant factor.  This fact about crossings goes back to an argument by
Newton (c.f. \cite{HLP34} section 2.22), establishing that the sequence
$a_k= \Pr[X=k]/\Pr[Y=k]$ is log-concave if $Y$ is binomially distributed.

Unfortunately, if $X$ and $Y$ are both generic sums of independent
Bernoullis with arbitrary means (as in our setting), then such a bound is
not known in the literature, and indeed it seems that essentially nothing
is known about the relation between the two distances in this general
case.  Without such a bound, it is rather unclear whether a number of
samples that does not scale with $n$ suffices to accurately learn in total
variation distance. Nevertheless, we extend the Kolmogorov distance
learning algorithm to total variation distance via a delicate algorithm
that exploits the detailed structure of a small $\eps$-cover
(\cite{DP:oblivious11,Daskalakis:anonymous08full}) of the space of all distributions
that are sums of independent Bernoulli random variables (see
Theorem~\ref{thm: sparse cover theorem}). Interestingly, this becomes
feasible by establishing an analog of the aforementioned argument by
Newton to a class of distributions used in the cover that are called {\em
heavy} (see Lemma~\ref{lem: from kolmogorov to TV}). This in turn relies on
probabilistic approximation results via {\em translated Poisson
distributions} (see
Definition~\ref{def:TPD},~\cite{Rollin:translatedPoissonApproximations}).

%To extend this preliminary result (learning with respect to Kolmogorov
%distance) to our desired learning %model (learning with respect to total
%variation distance), we develop a more sophisticated algorithm %that
%exploits the structure of a known small $\eps$-cover
%\cite{DP:oblivious,Daskalakis08wine} of the %space of all sums of
%Bernoulli random variables.

\medskip

\noindent {\bf Learning sums of Bernoullis in sublinear time.}
While the $\poly(1/\eps)$ sample complexity of Theorem~\ref{thm:sob} is essentially optimal\footnote{An easy reduction to distinguishing a fair coin from
an $\eps$-biased coin shows that any learning algorithm
for this problem needs $\Omega(1/\eps^2)$ samples.}, the running time is $\poly(n)$ (and super$\poly(1/\eps)$).
Moreover, generating a single sample from the hypothesis product
distribution requires $\Omega(n)$ uniformly random bits.
In general, any {\em proper} learning algorithm that explicitly
outputs a hypothesis vector $\hat{p} \in [0,1]^n$ will take
$\Omega(n)$ running time, and generating a sample from the hypothesis
$f(\hat{p})$ will require $\Omega(n)$ random bits.

More broadly, any algorithm (not necessarily proper) for learning
sums of Bernoullis
must have running time $\Omega((1/\eps^2)\cdot \log n)$ in the bit model
(since each sample is an $\Omega(\log n)$ bit string
and $\Omega(1/\eps^2)$ samples are needed).
Generating a sample from an arbitrary hypothesis distribution will in general require
$\Omega(\log n)$ bits
(since the entropy of the binomial distribution is $\Omega(\log n)$).
Hence a natural goal is to have a learning algorithm that runs in
$\poly(\log n, 1/\eps)$ time and requires $O(\log n)$  bits of randomness to generate
a draw from its hypothesis distribution. As discussed in the previous paragraph,
such an algorithm needs to be non-proper.

We show that there is an algorithm that satisfies all of these efficiency considerations, at the
cost of a $\log n$ factor in the sample complexity:

\begin{theorem} [Learning sums of Bernoullis in $\polylog(n)$ time with
efficient hypotheses] \label{thm:sob2}
 Let $f(x) = \sum_{i=1}^n x_i$.  There is an algorithm that uses
$\log(n) \cdot \poly(1/\eps)$ samples from the target distribution $f(\overline{p})$,
performs $\log^2(n) \cdot \poly(1/\eps)$ bit operations, and with probability at least $9/10$ outputs a
(succinct representation of a) hypothesis distribution $\mathcal{H}$ over $\{0,1,\ldots,n\}$
such that $\dtv(f(X),\mathcal{H}) \leq \eps.$  Moreover, a draw from the hypothesis distribution $\mathcal{H}$ can be
obtained in $\poly(\log n, 1/\eps)$ time using $O(\log(n/\eps))$ bits of randomness.
\end{theorem}

%are obvious targets for
%improvement; this is especially evident on observing that each sample drawn
%from $f(\overline{p})$ is an integer between 0 and $n$ and hence can be encoded by a
%$\log n$ bit string.  Note that
%even a simple hypothesis vector such as $\hat{p}= (1/2,\dots,1/2)$ does not meet
%this criterion: obtaining a single sample drawn from this distribution $f(\hat{p})$ (i.e. the Binomial distribution $B(n,1/2)$) requires $n$
%independent uniform random bits and hence linear time.

The key to Theorem~\ref{thm:sob2} is the simple observation that any sum of Bernoullis is a unimodal
distribution over the domain $\{0,1,\dots,n\}$.  This lets us apply a powerful algorithm due to \cite{Birge:97} that can learn any unimodal distribution to accuracy $\eps$
using $O(\log n)/\eps^3$ samples.  The algorithm outputs a hypothesis distribution that is a histogram over
$O((\log n)/\eps)$ intervals that cover $\{0,\dots,n\}$: more precisely, the hypothesis is uniform within each interval,
and for each interval the total mass it assigns to the interval is simply the fraction of
samples that landed in that interval. Thus, the hypothesis distribution has a succinct description and can be efficiently evaluated
using a small amount of randomness. We give details in Appendix~\ref{ap:birge}.

% Please consider rephrasing the sentence below.

We remark that by applying the algorithm of Theorem~\ref{thm:sob2} to the hypothesis distribution
that is provided by Theorem~\ref{thm:sob},
one can obtain a $2\eps$-accurate hypothesis satisfying the efficiency conditions of
Theorem~\ref{thm:sob2} (i.e. the hypothesis can be evaluated in poly-logarithmic time
and a sample from it can be generated using $O(\log(n/\eps))$ random bits).
Thus, it is possible to construct an $\eps$-accurate \emph{efficient} hypothesis using
$\poly(1/\eps)$ samples and $\poly(n)$ time. Whether this can be achieved in
\emph{poly-logarithmic} time is an interesting and challenging open problem; we discuss
this and other questions for future work in Section~\ref{sec:conclusion}.

\ignore{

In general the algorithm of Theorem~\ref{thm:log-cover-size} can
have running time exponential in $n$ is not computationally
efficient in general.  However, we can show that it can be implemented
efficiently in some interesting cases.  By exploiting the explicit
$\eps$-cover for the space of all sums of Bernoulli random variables
mentioned above (Theorem~\ref{thm: sparse cover theorem}), we obtain a
computationally efficient algorithm with small -- but non-constant --
sample complexity for a broader class of linear transformation functions
than just $f(x) = \littlesum_{i=1}^n x_i$:

}

\section{Preliminaries} \label{sec:prelim}

Recall that the {\em total variation distance} between two distributions
$\mathbb{P}$ and $\mathbb{Q}$ over a finite domain $D$ is
$\dtv\left(\mathbb{P},\mathbb{Q} \right):= (1/2)\cdot \littlesum_{\alpha
\in D}{|\mathbb{P}(\alpha)-\mathbb{Q}(\alpha)|}.$ Similarly, if $X$ and
$Y$ are two random variables ranging over a finite set, their total
variation distance $\dtv(X,Y)$ is defined as the total variation
distance between their distributions. Another notion of distance between
distributions/random variables that we use is the {\em Kolmogorov
distance}. For two distributions $\mathbb{P}$ and $\mathbb{Q}$ supported
on $\R$, their Kolmogorov distance is 
$\dk\left(\mathbb{P},\mathbb{Q} \right):= \sup_{x \in \R} \left|
\mathbb{P}( (-\infty, x])-\mathbb{Q}( (-\infty, x]) \right|.$ Similarly,
if $X$ and $Y$ are two random variables ranging over a subset of $\R$,
their Kolmogorov distance, denoted $\dk(X,Y),$ is the Kolmogorov distance
between their distributions. If two distributions $\mathbb{P}$ and
$\mathbb{Q}$ are supported on a finite subset of $\R$ we obtain
immediately that $\dk\left(\mathbb{P},\mathbb{Q}
\right) \le 2 \cdot \dtv\left(\mathbb{P},\mathbb{Q} \right).$

Fix a finite domain $D$, and let ${\cal P}$ denote some set of
distributions over $D.$ Given $\delta > 0$, a subset ${\cal Q} \subseteq
{\cal P}$ is said to be a \emph{$\delta$-cover of ${\cal P}$} (w.r.t.
total variation distance) if for every $\mathbb{P}$ in ${\cal P}$ there
exists some $\mathbb{Q}$ in ${\cal Q}$ such that
$\dtv(\mathbb{P},\mathbb{Q}) \leq \delta.$

We write $\sob = \sob_n $ to denote the set of all product distributions
over $\{0,1\}^n$, and $f(\sob)$ to denote the set of all transformed
product distributions $\{f(\overline{p})\}_{\overline{p} \in \sob}.$ We
write $f(\overline{p})(\alpha)$ to denote the probability of outcome
$\alpha$ under distribution $f(\overline{p}).$ Finally, for $\ell \in
\mathbb{Z}^+$ we write $[\ell]$ to denote $\{1,\dots,\ell\}$.

\section{A generic algorithm for any $f$ and some lower bounds} \label{sec:info}

In this section we give a simple generic algorithm to solve the transformed product distribution learning problem for any transformation function $f$, and state some lower bounds showing that even for rather simple functions $f$, the time and sample complexity of the generic algorithm may be
essentially the best possible.

\subsection{A generic algorithm} \label{subsec:genericalg}

The key ingredient in the generic algorithm is the following:

%\vspace{-0.2cm}

\begin{theorem} \label{thm:log-cover-size}
Fix a function $f: \{0,1\}^n \to \Omega$ where $\Omega$ is any range set. Suppose there exists a $\delta$-cover for $f(\cal{S})$ of size $N = N(n,\delta)$. Then
there is an algorithm that uses $O(\delta^{-2}\log N)$ samples and solves the $f$-transformed product distribution learning problem to accuracy $6\delta$.
\end{theorem}

The high-level idea behind Theorem~\ref{thm:log-cover-size} is as follows: for a pair of distributions $\bbQ_1, \bbQ_2 \in \mathcal{S}$, we define a {\em competition between $\bbQ_1$ and $\bbQ_2$} that takes as input a sample from the target distribution $f(X)$  and either crowns one of $\bbQ_1,\bbQ_2$ as the winner of the competition or calls the competition a draw.  Let $\mathcal{Q} \subseteq
\mathcal{S}$ be a $\delta$-cover for $f(\mathcal{S})$ of cardinality $N = N(n, \delta)$. The algorithm performs a tournament
between every pair of distributions from ${\cal Q}$ and outputs a distribution $\bbQ^{\ast} \in {\cal Q}$ that was never a loser, i.e. won or achieved a draw in all competitions. (If no such distribution exists, the algorithm outputs ``failure.'')

This basic approach of running a tournament between distributions in an $\delta$-cover is quite similar
to the algorithm of Devroye and Lugosi for choosing a density estimate (see \cite{DL96,DL97}
and Chapters~6 and~7 of \cite{DL:01}), which in turn built closely on the work of \cite{Yatracos85}.  Our algorithm achieves essentially the same bounds as these earlier approaches but there are some small
differences. (The DL approach uses a notion of the ``competition'' between two
tournaments which is not symmetric under swapping the two competing tournaments, whereas
our competition is symmetric; also, the DL approach chooses a distribution which wins the
maximum number of competitions as the output distribution, whereas our algorithm chooses
a distribution that is never defeated.)  We give our proof of Theorem~\ref{thm:log-cover-size}
in Appendix~\ref{ap:tournament}.

Theorem~\ref{thm:upper} is an easy
consequence of Theorem~\ref{thm:log-cover-size}.  Recall Theorem~\ref{thm:upper}:

\medskip

\noindent {\bf Theorem~\ref{thm:upper} (Information-theoretic upper bound for any $f$)}  \emph{
Let $f: \{0,1\}^n \to \Omega$ be an arbitrary function where
$\Omega$ is any range set.  There is an algorithm
that uses $O((n/\eps^2) \cdot \log(n/\eps))$ samples from the target
distribution $f(\overline{p})$, runs in time $(n/\eps)^{O(n)}$, and with probability at least $9/10$
outputs a hypothesis vector $\hat{p}$ such that
$\dtv(f(\overline{p}),f(\hat{p})) \leq \eps.$
}

\medskip

\begin{proof}
We argue that for any $f$ there is a $\delta$-cover for $f(\sob)$ of size at most $(n/\delta)^n$. The desired result
then follows from Theorem~\ref{thm:log-cover-size}.  Since for any pair of distributions $\bbP,\bbQ \in \sob$ and any function $f$ we
have $\dtv(f(\bbP), f(\bbQ)) \leq \dtv(\bbP,\bbQ)$, it suffices to exhibit a $\delta$-cover of the desired
cardinality for $\sob$.

We claim that if we discretize each individual expectation of our input Bernoulli random variables to integer multiples of $\alpha := {\delta \over n}$,
%(i.e. restrict the expectation of each input Bernoulli to the set $G = \{\alpha, 2\alpha,\ldots,
%\left \lfloor1 \over \alpha \right \rfloor \alpha \}$),
we obtain a $ \delta$-cover for $\sob$. Let us call the set of all such discretized product
distributions $\cal Q$. Clearly, $|{\cal Q}| \le \left( {n \over \delta} \right)^n$. Let $\bbP = (\bbP_1, \ldots, \bbP_n) \in \sob$.
Consider a point $\bbQ = (\bbQ_1, \ldots, \bbQ_n) \in \mathcal{Q}$ such that $\dtv(\bbQ_i, \bbP_i) \leq \delta/n$ for all $i$.
%Note that such a $\bbQ$ exists by construction.
Since both $\bbP$ and $\bbQ$ are product distributions, we have that $\dtv(\bbP, \bbQ) \leq
\littlesum_{i=1}^n \dtv(\bbQ_i, \bbP_i) \leq \delta$. This completes the proof.
\end{proof}

\subsection{Learning transformed product distributions can be computationally hard} \label{sec:nphard}

Though Theorem~\ref{thm:upper} shows that any learning problem $f$
in our framework can be solved with $\tilde{O}(n)$ sample complexity, it is natural to expect that some
learning problems can be \emph{computationally} hard.  We confirm this
intuition by establishing an NP-hardness result for a specific function $f$ that is computed by an explicit degree-2 polynomial.  We show that if there is a poly$(n)$-time algorithm for the transformed product distribution learning problem for this $f$, even for learning to constant accuracy, then $NP \subseteq BPP$.  Recall Theorem~\ref{thm:nphard}:

\medskip

\noindent {\bf Theorem~\ref{thm:nphard}}  \emph{
Suppose NP $\not \subseteq$ BPP.  Then there is an explicit
degree-2 polynomial $f$
such that there is no polynomial-time algorithm
that solves the transformed product distribution learning problem for $f$ to accuracy
$\eps = 1/3.$}

\medskip

The proof is a reduction from the PARTITION problem and is given in Appendix~\ref{ap:nphard}.

\subsection{Linear transformation functions can require $\Omega(n)$ samples} \label{subsec:samplelower}

We now show that even for a simple linear transformation function $f(x) = w \cdot x$ with small
integer weights, it can be impossible to significantly improve on the $\tilde{O}(n)$
sample complexity of the generic algorithm from Theorem~\ref{thm:upper}.  Recall Theorem~\ref{thm:linearlower}:

\medskip

\noindent {\bf
Theorem~\ref{thm:linearlower} (Sample complexity lower bound)}
\emph{
Fix any even $k \leq n$ and
let $f(x) = \sum_{i=k/2+1}^k i x_i.$ Let $L$ be any learning algorithm
that outputs a hypothesis vector $\hat{p}$ such that
$\dtv(f(\overline{p}),f(\hat{p})) \leq 1/40$ with probability at least
$e^{-o(k)}.$ Then $L$ must use $\Omega(k)$ samples from $f(\overline{p}).$}

\medskip

Theorem~\ref{thm:linearlower} is proved in Appendix~\ref{sec:infolower}.

%\newpage

\section{Learning an unknown sum of Bernoullis from $\poly(1/\eps)$ samples} \label{sec:SOB}

\subsection{Learning with respect to Kolmogorov distance} \label{sec:kolmogorov}

Let $X$ be any random variable supported on $\{0,1,\dots,n\}.$ We write  $F_X$ and $f_X$ to denote respectively the cumulative
distribution and  the probability density function of $X$.

Let $Z_1,\ldots,Z_k$ be independent samples of the random variable $X$, and define $Z_i^{(\ell)}:=\ind_{Z_i \le \ell}$ for all
$\ell=0,\ldots,n$ and $i=1,\ldots,k$. Clearly we have $\mathbf{E}\big[\littlesum_iZ_i^{(\ell)} / k\big] = F_X(\ell)$, which
suggests that $\hat{F}_X(\ell) := {\sum_iZ_i^{(\ell)} / k}$ may be a good estimator of $F_X(\ell)$ for all values of $\ell$, if
$k$ is large enough.  The Dvoretzky-Kiefer-Wolfowitz inequality (\cite{DKW56,Massart90}) confirms this, and in fact
shows that a surprisingly small value of $k$ -- independent of $n$ -- suffices.  The bound on $k$ given below is optimal up to
constant factors.

\begin{theorem}[DKW Inequality] \label{thm:cumulative distribution approximation}
Let $k=\max\{576,(9/8)\ln(1/\delta)\}\cdot (1/\eps^2).$  Then with probability at least $1 - \delta$ we have $\max_{0\le \ell \le n}\big|\hat{F}_X(\ell) - {F}_X(\ell)\big| \le \epsilon.$
\end{theorem}
In Appendix~\ref{ap:DKW} we give a self-contained proof of the theorem using elementary techniques (martingales and the method of bounded differences) and an interesting trick that goes back to Kolmogorov (see \cite{Peres:Soda09}). We start by defining a coupling between the process of learning the cumulative distribution function as our samples are revealed and a random walk on the line. Then Kolmogorov's trick is invoked to get a handle on the maximum estimation error, proving a weaker version of the theorem in which $k$ equals $\Theta({\frac 1
{\delta \eps^2}}).$ We then apply McDiarmid's inequality to bootstrap the weaker bound and obtain the tighter bound.

Now we specialize to the case in which $X=\sum_{i=1}^n X_i$ is a sum of independent Bernoulli random variables.
We use the DKW inequality to prove the following:

\begin{theorem}[Proper Learning under Kolmogorov Distance] \label{thm:proper learning kolmogorov}
Let $X=\sum_{i=1}^n X_i$ be a sum of independent Bernoulli random variables.
There is an algorithm which, given $k=\max\{9216,18\ln(1/\delta)\}\cdot (1/\eps^{2})$ independent samples from $F_X$,
produces with probability at least $1-\delta$ a set of independent Bernoullis $Y_1,\ldots,Y_n$ such that $d_K(X,Y) \le \epsilon$, where
$Y:=\sum_{i=1}^n Y_i$. The running time of the algorithm is ${\rm poly}(n/\eps) + n \cdot ({1 \over \epsilon})^{O(\log^2 {1 \over
\epsilon})}$.
\end{theorem}

\begin{prevproof}{Theorem}{thm:proper learning kolmogorov}
Use Theorem~\ref{thm:cumulative distribution approximation} to produce an $\epsilon \over 4$-approximation
$\{\hat{F}_X(\ell)\}_{\ell}$ of the cumulative distribution of $X$.  Theorem~\ref{thm: sparse cover theorem} below gives us that for all $\gamma
>0$, there exists a $\gamma$-cover in total variation distance of the set of all sums of $n$ Bernoulli random variables that
has size ${\rm poly}(n/\gamma) + n \cdot ({1 \over \gamma})^{O(\log^2 {1 \over \gamma})}.$\ignore{ (see Theorem~\ref{thm: sparse cover theorem}
for a more detailed description of the structure of the cover).}
 Construct such a cover using $\gamma=\epsilon/8$. Given that
the Kolmogorov distance between two distributions is always at most twice their total variation distance, this cover is in fact
a $\epsilon/4$-cover in the Kolmogorov distance. Output any $Y=\sum_{i=1}^n Y_i$ in the cover whose cumulative distribution
$F_Y$ satisfies

%\vspace{-0.7cm}

\begin{align}\max_{0\le \ell \le n}|F_Y(\ell) - \hat{F}_X(\ell)| \le \epsilon/2. \label{eq:target eq}
\end{align}

%\vspace{-0.2cm}

It is easy to see that a $Y$ satisfying~\eqref{eq:target eq} exists in the cover. Indeed, if $Y$ is the closest point of the
cover to $X$ in Kolmogorov distance, then it must be that $\max_{0\le \ell \le n}|F_Y(\ell) - {F}_X(\ell)| \le \epsilon/4.$
Given that $\{\hat{F}_X(\ell)\}_{\ell}$ is an $\epsilon/4$-approximation to $F_X$ the above inequality implies~\eqref{eq:target
eq}.

Moreover, it is easy to check that any $Y$ satisfying~\eqref{eq:target eq} will satisfy $\max_{0\le \ell \le n}|F_Y(\ell) -
{F}_X(\ell)| \le 3\epsilon/4 < \epsilon,$ using again that $\{\hat{F}_X(\ell)\}_{\ell}$ is an $\epsilon/4$-approximation to
$F_X$. Hence, we have $d_K(X,Y) < \epsilon.$
\end{prevproof}

\subsection{From Kolmogorov distance to total variation distance}

\ignore{The algorithm described in the previous subsection produces a sequence of Bernoulli random variables $Y_1,\ldots,Y_n$ whose sum
$f(Y)=\sum_{i=1}^n Y_i$ satisfies $\dk(X,Y) \leq \eps$, where $X=\sum_{i=1}^n X_i$ is the target sum of Bernoullis.}
The algorithm of the previous subsection learns the target sum of Bernoullis  to high accuracy with respect to Kolmogorov distance. Ideally, we would like to use this approximation as a black box to obtain an approximation in total variation distance.  As we discussed in the introduction, this runs to a very basic, apparently unresolved question in probability theory: is there a bound on the total variation distance between two sums of independent indicators in terms of their Kolmogorov distance? If at least one of the two sums is a Binomial distribution, then an argument due to Newton gives a positive answer. However nothing is known about the relation between the two distances in this general case.  Without such a bound, it is rather unclear whether constantly many samples (independent of $n$) suffices to accurately learn in total variation distance...

Nevertheless, we manage to extend our Kolmogorov distance learning algorithm  to the total variation distance
via a delicate algorithm that exploits the structure of a small $\eps$-cover (in total variation distance) of the space of all distributions that are sums of independent Bernoulli random variables. Interestingly, this becomes feasible by establishing an analog of the aforementioned argument by Newton to a class of distributions used in the cover that are called {\em heavy}; and this argument relies on probabilistic approximation results via {\em translated Poisson distributions} (see Definition~\ref{def:TPD}).

We give more details below.
%
%
%In this section we give an algorithm that learns $X$ to high accuracy with respect to total variation distance.  As mentioned in the Introduction, the algorithm uses the structure of a small $\eps$-cover of the space of all sums of independent Bernoulli random variables.
Let us start by formally stating a theorem that defines a cover (in total variation distance) of the space of sums of independent indicators.  The following is Theorem~9 of the full version of
 (\cite{DP:oblivious11}):

\begin{theorem} [Cover for sums of Bernoullis] \label{thm: sparse cover theorem}
%Let  ${\cal S} := \left\{ \{X_i\}_{i=1}^n~\vline~X_1,\ldots, X_n\text{ are independent indicator %r.v.'s}\right\}.$
For all $\epsilon >0$, there exists a set ${\cal S}_{\epsilon} \subseteq {\cal S}$ such that
%\begin{itemize}
%\item
(i)
$|{\cal S}_{\epsilon}| \le n^3 \cdot O(1/\epsilon) + n \cdot \left({1 \over \epsilon}\right)^{O(\log^2{1/\epsilon})}$;
%\item
(ii) For every $\{X_i\}_i \in {\cal S}$ there exists some $\{Y_i\}_i  \in {\cal S}_{\epsilon}$ such that $\dtv(\sum_i
X_i,\sum_i Y_i) \le \epsilon$ (i.e. $f(\sob_\eps)$ is an $\eps$-cover of $f(\sob)$); and
%\item
(iii) the set ${\cal S}_{\epsilon}$ can be constructed in time $O\left(n^3 \cdot O(1/\epsilon) + n \cdot \left({1 \over
\epsilon}\right)^{O(\log^2{1/\epsilon})}\right)$.
%\end{itemize}
Moreover, if  $\{Y_i\}_i  \in {\cal S}_{\epsilon}$, then the collection $\{Y_i\}_i$ has one of the following forms, where
$k=k(\epsilon) = O(1/\epsilon)$ is a positive integer: %\vspace{-0.2cm}
\begin{itemize}
\item (Sparse Form) There is a value $ \ell \leq k^3=O(1/\epsilon^3)$
such that
%\begin{itemize}
%\item
for all $i \leq \ell$ we have $\E[{Y_i}] \in \left\{{1 \over k^2}, {2\over k^2},\ldots, {k^2-1 \over k^2 }\right\}$, and
%\item
for all $i >\ell $ we have $\E[{Y_i}] \in \{0,  1\}$.
%\end{itemize}
% \vspace{-0.2cm}
\item ($k$-heavy Binomial Form) There is a value $\ell \in \{0,1,\dots,n\}$
and a value $q \in \left\{ {1 \over kn}, {2 \over kn},\ldots, {kn-1 \over kn}
\right\}$ such that
%\begin{itemize}
%\item
for all $i \leq \ell$ we have $\E[{Y_i}] = q$;
%\item
for all $i >\ell$ we have $\E[{Y_i}] \in \{0,  1\}$; and $\ell,q$ satisfy the bounds
%\item
$\ell q \ge k^2-{1\over k}$ and
%\item
$\ell q(1-q) \ge k^2- k-1-{3\over k}.$
%\end{itemize}
\end{itemize}
%Finally, for all $\{X_i\}_i \in {\cal S}$ for which there is no $\epsilon$-neighbor in ${\cal S}_{\epsilon}$ that is in sparse form, there exists a collection $\{X_i'\}_i \in {\cal S}$ in binomial form such that $d_{TV}(\sum_i X_i, \sum_i X'_i) \le \epsilon$.
%there exists a collection $\{X_i'\}_i \in {\cal S}$ and a collection $\{Y_i\}_i \in {\cal S}_{\epsilon}$ in binomial form such that the following are true:
%\begin{enumerate}
%\item[(a)] $d_{TV}(\sum_i X_i, \sum_i X'_i) \le \epsilon/2$;
%\item[(b)] $d_{TV}(\sum_i X'_i, \sum_i Y_i) \le \epsilon/2$;
%\item[(c)] letting $Z_{X'}, Z_Y \subseteq [n]$ denote the indices of variables in the collections $\{X'_i\}_i$ and $\{Y_i\}_i$ respectively that are deterministically zero, and $O_{X'}, O_Y$ the indices of variables that are deterministically $1$, $|Z_{X'}| = |Z_Y|$ and $|O_{X'}| = |O_{Y}|$.
%\end{enumerate}
\end{theorem}

\noindent
(We remark that \cite{Daskalakis:anonymous08full} establishes the same theorem, 
except that the size of the cover given there, as well as the time needed to produce it, are  $n^3 \cdot O(1/\epsilon) + n \cdot \left({1 \over \epsilon}\right)^{O({1/\epsilon^2})}$. Indeed, this weaker bound is obtained by enumerating over all possible collections $\{Y_i\}_i$ in sparse form and all possible collections in $k$-heavy Binomial Form, for $k=O(1/\epsilon)$ specified by the theorem.)

\medskip Using the cover described in Theorem~\ref{thm: sparse cover theorem}, we prove the following,
which immediately gives Theorem~\ref{thm:sob}:
\begin{theorem} [Learning under Total Variation Distance] \label{thm:proper learning tv}
Let $X=\sum_{i=1}^n X_i$ be a sum of independent Bernoullis.  Fix any $\tau>0$. There is an algorithm which, given $O({1 \over \epsilon^{8+\tau}})$ independent samples from $X$, produces
with probability at least $9/10$ a list of Bernoulli random variables $Y_1,\ldots,Y_n$ such that $\dtv(X,Y) \le \epsilon$,
where $Y:=\littlesum_{i=1}^n Y_i$. The running time of the algorithm is $n^3 \cdot \poly(1/\epsilon) +
n \cdot (1/\epsilon)^{O(\log^2{1/\epsilon})}$.
\end{theorem}

We first give a high-level outline of our argument.  The proof works by considering the points in a cover $\sob_{\eps^\beta}$ where $\beta$ is some constant $>1.$  We define two tests
that can be performed on points (i.e. distributions) in $\sob_{\eps^\beta}$.  The first of these, called the $\Delta$-test, is run on every sparse form distribution in $\sob_{\eps^\beta}$, and is designed to identify a sparse form distribution that is close to $X$ if such a
distribution exists.  The second test, called the $H$-test, is run on every $k$-heavy Binomial form distribution in $\sob_{\eps^\beta}$ and is designed to identify a Binomial form distribution that is close to $X$ if such a distribution exists.  Since $\sob_{\eps^\beta}$ is a cover, some test will succeed.  (Of course, we must also show that for each test, any distribution it outputs is indeed legitimately close to $X$; this is part of our analysis as well.)

We now enter into the detailed proof.
Let $\beta = 1 + {\frac {\tau}{12}}$ and let $\alpha = 4 + {\frac \tau 2}$.
 Using Theorem~\ref{thm:cumulative distribution approximation}, from $O(\epsilon^{-2 \alpha} )$ independent samples of $X$ we can obtain estimates $\{\hat{F}_X(\ell)\}_{\ell=0}^n$ such that $|\hat{F}_X(\ell) - {F}_X(\ell)| \le \epsilon^{\alpha}$, for all $\ell$, with probability at least $9/10$. For the rest of the proof we condition on the event that each of our estimates $\hat{F}_X(\ell)$ is indeed within  $\epsilon^\alpha$ of the actual value ${F}_X(\ell)$.
 Define $\hat{f}_X(0) = \hat{F}_X(0)$ and $\hat{f}_X(z)= \hat{F}_X(z)-\hat{F}_X(z-1)$, for all $z \in [n]$.

\subsubsection{Handling sparse form distributions in the cover.}
Let $Y=\sum_i Y_i$, where ${\cal Y}:=\{ Y_i \}_{i=1}^n \in {\cal S}_{\epsilon^{\beta}}$, and suppose that ${\cal Y}$ is of the {\em sparse form}, as defined in statement of Theorem~\ref{thm: sparse cover theorem}. Let ${\rm supp}(Y)$ denote the support of $Y$. By the definition of the sparse form, we have that $|{\rm supp}(Y)| \le (c \epsilon)^{-3 \beta}$, where $c$ is some universal constant. Define $\Delta_Y$ as follows:

%\vspace{-0.1cm}
$$\Delta_Y = {1 \over 2} \Big( \littlesum_{z \in {\rm supp}(Y)}  |f_Y(z) - \hat{f}_X(z)| +1- \littlesum_{z \in {\rm supp}(Y)} |\hat{f}_X(z)| \Big).$$

%\vspace{-0.1cm}

\noindent We observe that given $\{\hat{F}_X(\ell)\}_{\ell=0}^n$ and $Y$, the value of $\Delta_Y$ can be straightforwardly
computed in time $\poly(1/\eps)$  using dynamic programming.

The following two claims (whose proofs we defer to Section~\ref{sec:DeltaY}) say that
if $\dtv(X,Y)$ is large (at least $\eps$) then $\Delta_Y$ must be fairly large, while if $\dtv(X,Y)$ is small (at most $\eps^\beta$) then $\Delta_Y$ must also be fairly small.

\begin{claim} \label{claim: if far in cover, then far in distribution approximation}
If $\dtv(X,Y) \ge \epsilon$, then $\Delta_Y \ge \epsilon -  2 (c \epsilon)^{-3 \beta} \cdot \epsilon^{\alpha}$.
\end{claim}

\begin{claim} \label{claim: if close in cover, then close in distribution approximation}
If $\dtv(X,Y) \le \epsilon^\beta$, then $\Delta_Y \le  \epsilon^\beta +  2 (c \epsilon)^{-3 \beta} \cdot \epsilon^{\alpha}$.
\end{claim}

By our choice of $\beta = 1 + {\frac {\tau}{12}}$ and $\alpha = 4 + {\frac \tau 2}$, for $\eps$ smaller than a certain constant
(depending on $c$ and $\tau$)  the following condition holds: %\vspace{-0.2cm}
\begin{align}
\epsilon > \epsilon^\beta + 4 c^{-3\beta} \epsilon^{\alpha-3\beta}. \label{eq:condition 1}
\end{align}

%\vspace{-0.1cm}

Claims~\ref{claim: if far in cover, then far in distribution approximation} and \ref{claim: if close in cover, then close in distribution approximation} imply that if we use the ${\cal S}_{\epsilon^\beta}$ cover of Theorem~\ref{thm: sparse cover theorem}, we can filter collections $\{ Y_i\}_i \in {\cal S}_{\epsilon^\beta}$ in the sparse form whose sum $Y$ is $\epsilon$-far in total variation distance from $X$, by computing $\Delta_Y$ and thresholding at the value $\epsilon^{\beta} + 2 (c \epsilon)^{-3 \beta} \cdot \epsilon^{\alpha}$. Moreover, this filtration is not going to get rid of any collections in sparse form that are within $\epsilon^\beta$ total variation distance from the target distribution. Formally, let us define the following test, which takes as input the estimates $\{\hat{F}_X(\ell) \}_{\ell=0}^n$ and decides whether or not to reject a point in ${\cal S}_{\epsilon^\beta}$ in sparse form.

\begin{definition}[$\Delta$-test]
The input is $\Y=\{Y_i\}_i \in {\cal S}_{\epsilon^\beta}$ in the sparse form.  Let $Y =\sum_i Y_i$.
If $\Delta_Y \le \epsilon^{\beta} + 2 (c \epsilon)^{-3 \beta} \cdot \epsilon^{\alpha}$ then
the $\Delta$-test accepts $\Y$, otherwise it rejects $\Y.$
\end{definition}

Since $\Delta_Y$ can be computed in $\poly(1/\eps)$ time, an execution of the $\Delta$-test can
be performed in $\poly(1/\eps)$ time.
If~\eqref{eq:condition 1} is satisfied, Claims~\ref{claim: if far in cover, then far in distribution approximation} and~\ref{claim: if close in cover, then close in distribution approximation} imply the following:
\begin{lemma}[Correctness of the $\Delta$-test]\label{lemma: correctness of delta test}
Let $\{ Y_i\}_i \in {\cal S}_{\epsilon^\beta}$ be in the sparse form and $Y=\sum_i Y_i$. If $Y$ is accepted by the $\Delta$-test then $\dtv(X,Y) \leq \eps,$ and if $\dtv(X,Y) \leq \eps^b$ then $Y$ is accepted by the $\Delta$-test.
%$\Delta_Y \le \epsilon^{\beta} + 2 (c \epsilon)^{-3 \beta} \cdot \epsilon^{\alpha}$, then
%$\dtv(Y, X) \le \epsilon$. Moreover, if $\dtv(Y, X) \le \epsilon^\beta$, then $\Delta_Y \le \epsilon^{\beta} %+ 2 (c \epsilon)^{-3 \beta} \cdot \epsilon^{\alpha}$.
\end{lemma}

\noindent Lemma~\ref{lemma: correctness of delta test} implies in particular that if $\{X_i\}_i$ has an $\epsilon^\beta$-neighbor $\{Y_i\}_i$ in ${\cal S}_{\epsilon^\beta}$ that is in the sparse form, then this neighbor will be accepted by the $\Delta$-test. Moreover, no element $\{Y_i\}_i \in {\cal S}_{\epsilon^\beta}$ in sparse form that is accepted by the $\Delta$-test has $\sum_i Y_i$ further than $\epsilon$ in total variation distance from $\sum_i X_i$.

\subsubsection{Handling Binomial form distributions in the cover.}

We can use the $\Delta$-test for the sparse points in the cover, but it could be that the target collection $\X=\{X_i\}_i$ has no sparse $\epsilon^\beta$-neighbor in ${\cal S}_{\epsilon^\beta}$ and the $\Delta$-test fails to accept any sparse point in the cover. We need to devise a procedure which similarly filters the points of heavy Binomial form in the cover so that we do not eliminate any $\epsilon^\beta$-close point, while at the same time not admitting any $\epsilon$-far point. Since $\X$ has no sparse $\epsilon^\beta$-neighbor in the cover, it follows from Theorem~\ref{thm: sparse cover theorem} that there is a collection $\X':=\{X_i'\}_i \in \sob_{\eps^\beta}$ in $k(\epsilon^\beta)$-heavy Binomial form such that $\sum_i X_i$ and $\sum_i X_i'$ are within $\epsilon^\beta$ in total variation distance.

We show that the total variation distance is essentially within a constant factor of the Kolmogorov distance for two collections of random variables in heavy Binomial form:

\begin{lemma}\label{lem: from kolmogorov to TV}
Let $\X:=\{X_i\}_i$ and $\Y:=\{Y_i\}_i$ be two collections of independent indicators in $k$-heavy Binomial form and set $X=\sum_i X_i$, $Y=\sum_i Y_i$. Then
${1 \over 2} d_K(X,Y) \le \dtv(X,Y) \le 2 \cdot d_K(X,Y) + O(1/k).$
\end{lemma}

The proof is somewhat lengthy so we defer it to Section~\ref{sec:kolmtoTV}.  Given Lemma~\ref{lem: from kolmogorov to TV}, we are inspired to define the {\em $H$-test} as follows. The test takes as input the estimates $\{\hat{F}_X(\ell)\}_{\ell=0}^n$ and needs to decide whether or not to reject a point in ${\cal S}_{\epsilon^\beta}$ in $k(\epsilon^\beta)$-heavy Binomial form.
\begin{definition}[$H$-test]
The input is $\Y=\{Y_i\}_i \in {\cal S}_{\epsilon^\beta}$ in the $k(\eps^\beta)$-heavy Binomial form.  Let $Y =\sum_i Y_i$. If
$\max_{0\le \ell \le n}|F_Y(\ell) - \hat{F}_X(\ell)| \le 2 \epsilon^\beta + \epsilon^\alpha$ then the $H$-test accepts $\Y$,
otherwise it rejects $\Y$.
\end{definition}

\noindent Like the $\Delta$-test, the $H$-test can be performed in poly$(1/\eps)$ time.  We now prove:

\begin{lemma}[Correctness of the $H$-test] \label{lem:correctness of H-test}
Suppose that $\X$ is not $\epsilon^\beta$-close to any point in ${\cal S}_{\epsilon^\beta}$ of the sparse form. Let $\Y=\{Y_i\}_i \in {\cal S}_{\epsilon^\beta}$ be of the $k(\epsilon^\beta)$-heavy Binomial form, and let $Y=\sum_i Y_i$. If $\Y$ is accepted by the $H$-test, then $\dtv(X,Y) \le 4\epsilon^\alpha + O(\epsilon^\beta).$
On the other hand, if $\dtv(X,Y)) \le \epsilon^\beta$, then $\Y$ is accepted by the $H$-test.
\end{lemma}

%\vspace{-0.1cm}

\begin{prevproof}{Lemma}{lem:correctness of H-test}
Since $\X$ is not $\eps^\beta$-close to any sparse point in the cover, it follows from Theorem~\ref{thm: sparse cover theorem}
that there exists a collection $\X':=\{X_i'\}_i \in {\cal S}$ in the $k(\epsilon^\beta)$-heavy Binomial form such that
$X:=\sum_i X_i$ and $X':=\sum_i X_i'$ are within $\epsilon^\beta$ in total variation distance.

Suppose that $\Y$ passes the $H$-test. For all $\ell$, we have $|F_Y(\ell) - F_X(\ell)| \le |F_Y(\ell) -
\hat{F}_{X}(\ell)|+|\hat{F}_{X}(\ell) - F_X(\ell)|$, and hence
$
\dk(X,Y) \le 2 \epsilon^\beta + 2 \epsilon^\alpha.$
Given this, we have

%\vspace{-0.3cm}

\begin{align*}
\dtv(X,Y) &\le \dtv(X,X') + \dtv(X',Y)~~~~~~~~~~~~~~~~~~~~~~~(\text{using the triangle inequality})\\
&\le \epsilon^\beta + \dtv(X',Y)\\
&\le \epsilon^\beta + {2}\dk(X',Y) + O(1/k(\epsilon^\beta))~~~~~~~~~~~~~~~(\text{using Lemma~\ref{lem: from kolmogorov to TV}})\\
&\le {2}\dk(X',Y) + O(\epsilon^\beta)~~~~~~~~~~~~~~~~~~~~~~~~~~~~~~~~~(\text{since Theorem~\ref{thm: sparse cover theorem} gives }1/k(\epsilon^\beta) = O(\epsilon^\beta))\\
&\le {2}\dk(X',X) + {2}\dk(X,Y) + O(\epsilon^\beta)~~~~~~~~~~(\text{using the triangle inequality})\\
&\le {4}\dtv(X',X) + {2}\dk(X,Y) + O(\epsilon^\beta)~~~~~~~~(\text{using that $\dk \leq 2\cdot \dtv$ })\\
&\le {2}\dk(X,Y) + O(\epsilon^\beta) \le 4\epsilon^\alpha+ O(\epsilon^\beta).
\end{align*}

On the other hand, if $\dtv(X,Y)) \le \epsilon^\beta$, it follows that $\dk(X,Y)) \le 2 \epsilon^\beta$. Hence, for all $\ell$,
\begin{align*}
|F_Y(\ell) - \hat{F}_X(\ell)| &\le |F_Y(\ell) -{F}_{X}(\ell)|+|\hat{F}_{X}(\ell) - F_X(\ell)|\\
& \le \dk(X,Y) + \epsilon^\alpha
 \le 2 \dtv(X,Y) + \epsilon^\alpha
\le 2 \epsilon^\beta + \epsilon^\alpha.
\end{align*}
Hence, $\Y$ is accepted by the $H$-test.
\end{prevproof}

\subsubsection{Finishing the Proof of Theorem~\ref{thm:proper learning tv}}

Let $\hat{c}$ be the constant hidden in the $O(\cdot)$-notation in the statement of Lemma~\ref{lem:correctness of H-test}.  We may assume that $\eps$ is smaller than any fixed constant, and hence that it satisfies $
\epsilon \ge 4 \epsilon^\alpha +  \hat{c} \cdot \epsilon^{\beta}$ %, \label{eq:condition 2}
as well as~\eqref{eq:condition 1}.
%$\epsilon > \epsilon^\beta + 4 c^{-3\beta} \epsilon^{\alpha-3\beta}.$
We now describe the algorithm promised in Theorem~\ref{thm:proper learning tv}.  The algorithm takes as input the estimates $\{\hat{F}_X(\ell) \}_{\ell=0}^n$ (which, as described at the beginning of the proof, can be obtained from the samples from $X$ using Theorem~\ref{thm:cumulative distribution approximation}).

\medskip \noindent

\hskip-.2in \framebox{
\medskip \noindent \begin{minipage}{15cm}
{\sc Algorithm}
\begin{enumerate}
\item Compute the cover ${\cal S}_{\epsilon^\beta}$ defined in Theorem~\ref{thm: sparse cover theorem}.

\item If any $\Y \in {\cal S}_{\epsilon^\beta}$ in the sparse form passes the $\Delta$-test,
output such a $\Y$ and halt.

\item Otherwise, if any $\Y \in {\cal S}_{\epsilon^\beta}$ in the $k(\epsilon^\beta)$-heavy Binomial form passes the $H$-test, output such a $\Y$.
\end{enumerate}
\end{minipage}}

\medskip It follows from Theorem~\ref{thm: sparse cover theorem}
that there exists some $\{Y_i\}_i \in  {\cal S}_{\epsilon^\beta}$ such that $\sum_i Y_i$ is within $\epsilon^\beta$ in total
variation distance from $\sum_i X_i$. If there exists such an element in the cover that is also in the sparse form, it follows
from Lemma~\ref{lemma: correctness of delta test} that this point will pass the test at the second test of the algorithm and
hence be returned in the output. On the other hand, any element $\{Y_i\}_i$ of the cover returned by the second step of the
algorithm will satisfy that $\sum_i Y_i$ is within $\epsilon$ in total variation distance from $\sum_i X_i$. If the second step
of the algorithm fails to return any element of the cover, it follows that $\X$ has an $\epsilon^\beta$-neighbor in the cover
in heavy Binomial form. Lemma~\ref{lem:correctness of H-test} implies then that such a  neighbor will be output in the third
step of the algorithm. Moreover, the lemma implies that any element returned in the third step is an $\epsilon$-neighbor of
$\X$. Hence the algorithm is correct and always succeeds in returning an $\epsilon$-neighbor of $\X$. Finally, the running time
is dominated by the time to run the $\Delta$-test or the $H$-test on each point in the cover ${\cal S}_{\epsilon^\beta}$, which
is easily seen to be $n^3 \cdot \poly(1/\epsilon) + n \cdot ({1/\epsilon})^{O(\log^2{1/\epsilon})}.$ This concludes the proof of
Theorem~\ref{thm:proper learning tv} and thus also of Theorem~\ref{thm:sob}. \qed

\ignore{

An interesting goal for future work is to obtain a proper learning algorithm that has running time
dependence on $\eps$ that is $\poly(1/\eps)$, rather than quasipoly$(1/\eps)$ as in our
current algorithm.

}

\subsection{Proof of Claims~\ref{claim: if far in cover, then far in distribution approximation}
  and~\ref{claim: if close in cover, then close in distribution approximation}} \label{sec:DeltaY}

Recall that $Y=\sum_i Y_i$ where ${\cal Y}:=\{ Y_i \}_{i=1}^n \in {\cal S}_{\epsilon^{\beta}}$ is of the {\em sparse form}, as defined in statement of Theorem~\ref{thm: sparse cover theorem}.

Recall Claim~\ref{claim: if far in cover, then far in distribution approximation}:

\medskip

\noindent {\bf Claim~\ref{claim: if far in cover, then far in distribution approximation}}
\emph{
If $\dtv(X,Y) \ge \epsilon$, then $\Delta_Y \ge \epsilon -  2 (c \epsilon)^{-3 \beta} \cdot \epsilon^{\alpha}$.}

\medskip

\begin{prevproof}{Claim}{claim: if far in cover, then far in distribution approximation}
By the definition of the total variation distance, the hypothesis implies
\begin{align}
{1 \over 2}\sum_{z }  |f_Y(z) -{f}_X(z)| \ge \epsilon. \label{eq: lalala}
\end{align}
We can bound the left hand side of the above as follows
\begin{align}
\sum_{z }  |f_Y(z) -{f}_X(z)| &= \sum_{z \in {\rm supp}(Y)}  |f_Y(z) -{f}_X(z)| + \sum_{z \notin {\rm supp}(Y)}  |{f}_X(z)| \notag\\
&\le \sum_{z \in {\rm supp}(Y)}  |f_Y(z) -\hat{f}_X(z)| + \sum_{z \in {\rm supp}(Y)}  |f_X(z) -\hat{f}_X(z)| + \sum_{z \notin {\rm supp}(Y)}  |{f}_X(z)| \notag\\
&\le \sum_{z \in {\rm supp}(Y)}  |f_Y(z) -\hat{f}_X(z)| +2 (c \epsilon)^{-3 \beta} \cdot \epsilon^{\alpha}+ \sum_{z \notin {\rm supp}(Y)}  |{f}_X(z)|. \label{eq: lala}
\end{align}
In the last line of the above we used the bound on the support of $Y$ and the fact that for all $z \in [n]$:
\begin{align*}
|f_X(z) -\hat{f}_X(z)|&=|(F_X(z) - F_X(z-1)) - (\hat{F}_X(z) - \hat{F}_X(z-1)| \\
&\le |F_X(z)-\hat{F}_X(z)| + |F_X(z-1)-\hat{F}_X(z-1)| \le 2 \cdot \epsilon^{\alpha},
\end{align*}
while $|f_X(0) -\hat{f}_X(0)|=|F_X(0) -\hat{F}_X(0)| \le \epsilon^{\alpha}.$

Finally, note that
\begin{align}
 \sum_{z \in {\rm supp}(Y)}  |{f}_X(z)|  &= \sum_{z \in {\rm supp}(Y)}  |\hat{f}_X(z)- (\hat{f}_X(z) -f_X(z))| \notag\\
&\ge \sum_{z \in {\rm supp}(Y)}  (|\hat{f}_X(z)|- |\hat{f}_X(z) -f_X(z)|). \notag
\end{align}
Hence,
\begin{align}
 \sum_{z \notin {\rm supp}(Y)}  |{f}_X(z)|  &= 1- \sum_{z \in {\rm supp}(Y)}  |{f}_X(z)| \notag\\
 &\le 1- \sum_{z \in {\rm supp}(Y)}  |\hat{f}_X(z)| +\sum_{z \in {\rm supp}(Y)}  |\hat{f}_X(z) -f_X(z)| \notag\\
&\le 1- \sum_{z \in {\rm supp}(Y)}  |\hat{f}_X(z)| + 2 (c \epsilon)^{-3 \beta} \cdot \epsilon^{\alpha}. \label{eq:lalalall}
\end{align}
Using~\eqref{eq: lalala}, \eqref{eq: lala}, \eqref{eq:lalalall} we obtain that $\Delta_Y \ge \epsilon -  2 (c \epsilon)^{-3 \beta} \cdot \epsilon^{\alpha}.$
\end{prevproof}

Recall Claim~\ref{claim: if close in cover, then close in distribution approximation}:

\medskip

\noindent {\bf Claim~\ref{claim: if close in cover, then close in distribution approximation}}
\emph{
If $\dtv(X,Y) \le \epsilon^\beta$, then $\Delta_Y \le  \epsilon^\beta +  2 (c \epsilon)^{-3 \beta} \cdot \epsilon^{\alpha}$.}

\medskip

\begin{prevproof}{Claim}{claim: if close in cover, then close in distribution approximation}
By definition of the total variation distance, the hypothesis implies
\begin{align}
{1 \over 2}\sum_{z }  |f_Y(z) -{f}_X(z)| \le \epsilon^{\beta}. \nonumber
\end{align}
We can bound the left hand side of the above as follows
\begin{align}
\sum_{z }  |f_Y(z) -{f}_X(z)| &= \sum_{z \in {\rm supp}(Y)}  |f_Y(z) -{f}_X(z)| + \sum_{z \notin {\rm supp}(Y)}  |{f}_X(z)| \notag\\
&\ge \sum_{z \in {\rm supp}(Y)}  |f_Y(z) -\hat{f}_X(z)| - \sum_{z \in {\rm supp}(Y)}  |f_X(z) -\hat{f}_X(z)| + \sum_{z \notin {\rm supp}(Y)}  |{f}_X(z)| \notag\\
&\ge \sum_{z \in {\rm supp}(Y)}  |f_Y(z) -\hat{f}_X(z)| - 2 (c \epsilon)^{-3 \beta} \cdot \epsilon^{\alpha}+ \sum_{z \notin {\rm supp}(Y)}  |{f}_X(z)|. \notag
\end{align}
In the last line of the above we used the same bound we used for deriving~\eqref{eq: lala}.

Finally, note that
\begin{align}
 \sum_{z \notin {\rm supp}(Y)}  |{f}_X(z)|  &= 1- \sum_{z \in {\rm supp}(Y)}  |{f}_X(z)| \notag\\
 &\ge 1- \sum_{z \in {\rm supp}(Y)}  |\hat{f}_X(z)| -\sum_{z \in {\rm supp}(Y)}  |\hat{f}_X(z) -f_X(z)| \notag\\
&\ge 1- \sum_{z \in {\rm supp}(Y)}  |\hat{f}_X(z)| - 2 (c \epsilon)^{-3 \beta} \cdot \epsilon^{\alpha} \nonumber
\end{align}
Using the bounds above we obtain $\Delta_Y \le \epsilon^{\beta} + 2 (c \epsilon)^{-3 \beta} \cdot \epsilon^{\alpha}.$
\end{prevproof}

\subsection{Proof of Lemma~\ref{lem: from kolmogorov to TV}} \label{sec:kolmtoTV}

Recall Lemma~\ref{lem: from kolmogorov to TV}:

\medskip

\noindent {\bf Lemma~\ref{lem: from kolmogorov to TV}}  \emph{
Let $\X:=\{X_i\}_i$ and $\Y:=\{Y_i\}_i$ be two collections of independent indicators in $k$-heavy Binomial form and set $X=\sum_i X_i$, $Y=\sum_i Y_i$. Then
$${1 \over 2} d_K(X,Y) \le \dtv(X,Y) \le 2 \cdot d_K(X,Y) + O(1/k).$$
}

\begin{prevproof}{Lemma}{lem: from kolmogorov to TV}
The first inequality is immediate from the definition of the Kolmogorov and Total Variation distances. To show the other bound, let $Z_X, Z_Y \subseteq [n]$ be the indices of the variables in the collections $\{X_i\}_i$ and $\{Y_i\}_i$ respectively that are deterministically zero, $O_X, O_Y$ the indices of variables that are deterministically~$1$, and define $E_X = [n]\setminus Z_X \setminus O_X$, $E_Y=[n] \setminus Z_Y \setminus O_Y$, $n_1=|E_X|$, $n_2 = |E_Y|$, and $m = |O_X| - |O_Y|$. Moreover, let $p_1$ be the common mean of the variables $X_i, i \in E_X$, and $p_2$ the common mean of the variables in $Y_i, i\in E_Y$. Without loss of generality, we can assume that $m \ge 0$. Now we define $X' = m + {\rm Bin}(n_1,p_1)$ and $Y'= {\rm Bin}(n_2,p_2)$. It is straightforward to check that
\begin{align*}
\dtv(X,Y) = \dtv(X',Y')
\end{align*}
and
$$\dk(X,Y) = \dk(X',Y').$$
Given that $\X$ and $\Y$ are in $k$-heavy Binomial form, it follows that for $i=1,2$:
\begin{align*}
\mu_i:=n_i \cdot p_i &\ge k^2-{1\over k};\\
\sigma_i^2:=n_i \cdot p_i(1-p_i) &\ge k^2- k-1-{3\over k}.
\end{align*}

We recall that the {\em Translated Poisson distribution} is defined as follows.
\begin{definition}[\cite{Rollin:translatedPoissonApproximations}] \label{def:TPD}We say that an integer random variable $Y$ has a {\em translated Poisson distribution} with paremeters $\mu$ and $\sigma^2$ and write ${\cal L}(Y)=TP(\mu,\sigma^2)$
if ${\cal L}(Y - \lfloor \mu-\sigma^2\rfloor) = Poisson(\sigma^2+ \{\mu-\sigma^2\})$, where $\{\mu-\sigma^2\}$ represents the fractional part of $\mu-\sigma^2$.
\end{definition}

Given the above, and following~\cite{Daskalakis:anonymous08full} (see Section 6.1), we can show the following for $i=1,2$:
\begin{align}
\dtv\left({\rm Bin}(n_i,p_i),~ TP(\mu_i, \sigma_i^2) \right)= O(1/k). \label{eq:kourash1}
\end{align}
Now we show the following:
\begin{lemma}\label{lem:from Kolmogorov to TV for translated poissons}
For  $\lambda, \widehat{\lambda}>0$, $m, \widehat{m} \in \mathbb{N}_0$, let $Y = m + Poisson(\lambda)$ and $\widehat{Y} = \widehat{m} + Poisson(\widehat{\lambda})$. Then
$$\dtv(Y,\widehat{Y}) \le {2}d_K(Y,\widehat{Y}).$$
\end{lemma}
\begin{prevproof}{Lemma}{lem:from Kolmogorov to TV for translated poissons}
Without loss of generality assume that $m':=m - \widehat{m} \ge 0$. Then it is enough to compare $Y' = m' + Poisson(\lambda)$ and $\widehat{Y}' = Poisson(\widehat{\lambda})$, since $\dtv(Y,\widehat{Y}) = \dtv(Y',\widehat{Y}')$ and $\dk(Y,\widehat{Y}) = \dk(Y',\widehat{Y}')$. For $i\ge 0$, define
\begin{align*}
R_i :=  { \Pr[{Y}' = i] \over \Pr[\widehat{Y}' = i]}.
\end{align*}
Clearly, $R_i =0$, for $i=0, 1,\ldots, m'-1$, since $Y'$ is not supported on that set. On the other hand for all $i \ge m'$, we have
\begin{align*}
R_i :=  {{ \lambda^{i-m'} \cdot e^{-\lambda} \over (i-m')! } \over { \widehat{\lambda}^{i} \cdot e^{-\widehat{\lambda}} \over i! } },
\end{align*}
and for all $i \ge m'+1$:
\begin{align*}
{R_i \over R_{i+1} }= {\widehat{\lambda} \cdot (i+1-m') \over {\lambda} \cdot (i+1) }.
\end{align*}
\noindent Let us distinguish the following cases:
\begin{itemize}
\item If $\lambda > \widehat{\lambda}$, then ${R_i \over R_{i+1} } <1$ for all $i$. Hence, $R_i$ is increasing in $i$, so that it can change from a value $\le 1$ to a value $\ge 1$ at most one time. Hence, there exists a single $i^*$ such that $\Pr[{Y}' = i] <  \Pr[\widehat{Y}' = i]$ for all $i < i^*$, and $\Pr[{Y}' = i] <\Pr[\widehat{Y}' = i]$ for all $i > i^*$. In this case, it is easy to see that
$$\dtv(Y',\widehat{Y}') = \dk(Y',\widehat{Y}').$$
\item If $\lambda = \widehat{\lambda}$ and $m'=0$ the variables $Y'$ and $\widehat{Y}'$ are identically distributed so that their total variation distance and Kolmogorov distance are both identically $0$. If $\lambda = \widehat{\lambda}$  and $m' >0$, then ${R_i \over R_{i+1} } <1$ for all $i$ and from our argument in the previous case we obtain
$$\dtv(Y',\widehat{Y}') = d_K(Y',\widehat{Y}').$$
\item Finally, if $\lambda < \widehat{\lambda}$, then ${R_i \over R_{i+1} } \ge 1$ for all $i \le {m' \over 1-{\lambda \over \widehat{\lambda}} }-1$ and ${R_i \over R_{i+1} } > 1$ for all $i > {m' \over 1-{\lambda \over \widehat{\lambda}} }-1$. So $R_i$ is increasing up to some $i^*$ and decreasing above $i^*$. So the distributions of $Y'$ and $\widehat{Y}'$ have at most two intersections. Hence, we obtain
$$\dtv(Y',\widehat{Y}') \le {2}d_K(Y',\widehat{Y}').$$
\end{itemize}\end{prevproof}

Given the above we have,
\begin{align*}
\dtv(X',Y') &\equiv \dtv(m+{\rm Bin}(n_1,p_1),{\rm Bin}(n_2,p_2))\\
&\le \dtv\left(m+{\rm Bin}(n_1,p_1),~ m+TP(\mu_1, \sigma_1^2) \right) + \dtv\left(m+TP(\mu_1, \sigma_1^2),~ TP(\mu_2, \sigma_2^2) \right) \\
&~~~~~~~~~~~~~~~~~+\dtv\left({\rm Bin}(n_2,p_2),~ TP(\mu_2, \sigma_2^2) \right)~~~~~~~~~~~~~~~\text{(using the triangle inequality)}\\
&\le O(1/k) + \dtv\left(m+ TP(\mu_1, \sigma_1^2),~ TP(\mu_2, \sigma_2^2) \right)~~~~~~~~~~~~~\text{(using \eqref{eq:kourash1})}\\
&\le O(1/k) + 2 \dk\left(m+ TP(\mu_1, \sigma_1^2),~ TP(\mu_2, \sigma_2^2) \right)~~~~~~~~~~~~\text{(using Lemma~\ref{lem:from Kolmogorov to TV for translated poissons})}\\
&\le O(1/k) + 2 \dk\left(m+{\rm Bin}(n_1,p_1), m+TP(\mu_1, \sigma_1^2) \right) +2\dk\left({\rm Bin}(n_2,p_2),~ TP(\mu_2, \sigma_2^2) \right)\\
&~~~~~~~~~~~~~~~~~+ 2 \dk\left(m+{\rm Bin}(n_1,p_1),~ {\rm Bin}(n_2,p_2) \right) ~~~~~~~~~~~~~\text{(using the triangle inequality)}\\
&\le O(1/k) + 2 \dk\left(m+{\rm Bin}(n_1,p_1),~ {\rm Bin}(n_2,p_2) \right)~~~~~~~~~~~\text{(\text{using that $\dk \leq 2\cdot \dtv$ } and \eqref{eq:kourash1}). }\\
&= O(1/k) + 2 \dk\left(X',Y' \right)
\end{align*}
This concludes the proof of Lemma~\ref{lem: from kolmogorov to TV}.\end{prevproof}

\section{An intermediate case:  Linear transformation functions with $O(1)$ distinct weights}
\label{sec:constantnumofwts}

Recall that Theorem~\ref{thm:nphard}  shows that the exponential running time of Theorem~\ref{thm:log-cover-size} cannot be significantly improved even if the transformation function $f$ is a degree-2 polynomial, and Theorem~\ref{thm:linearlower}\  shows that the $\tilde{O}(n)$ sample complexity
cannot be significantly improved even if $f$ is a simple linear function.  In contrast with these strong
negative results, we have also seen that the sum-of-Bernoullis transformation function $f(x) = \sum_{i=1}^n x_i$ admits highly efficient algorithms both in terms of running time and sample complexity.  We close this paper
by showing that in an intermediate case -- if the transformation function $f$ is a linear function with
constantly many different weights -- then it is also possible to improve on the generic time and sample complexity bounds of Theorem~\ref{thm:log-cover-size}, though not quite as dramatically as for sums of Bernoullis:
\ignore{\footnote{The proof
of the following theorem, given in Appendix~\ref{sec:O(1)weights}, uses tools from Section~\ref{sec:SOB};
the reader is advised to read that section before the proof of the theorem.}}

\ignore{
However, we can show that it can be implemented
efficiently in some interesting cases.  By exploiting the explicit
$\eps$-cover for the space of all sums of Bernoulli random variables
mentioned above (Theorem~\ref{thm: sparse cover theorem}), we obtain a
computationally efficient algorithm with small -- but non-constant --
sample complexity for a broader class of linear transformation functions
than just $f(x) = \littlesum_{i=1}^n x_i$:
}

\begin{theorem} [Linear transformation functions with $O(1)$ different
weights]
\label{thm:linearupper} Let $f(x) = \sum_{i=1}^n a_i x_i$ be any function
such that there are at most $k$ different values in the set
$\{a_1,\dots,a_n\}.$ Then there is an algorithm that uses $k \log(n) \cdot
\widetilde{O}(\eps^{-2})$ samples from the target distribution
$f(\overline{p})$, runs in time $\poly(n^k \cdot
\eps^{-k\log^2(1/\eps)})$, and with probability at least $9/10$ outputs a
hypothesis vector $\hat{p}$ such that $\dtv(f(\overline{p}),f(\hat{p}))
\leq \eps.$ \end{theorem}

Note that setting $a_1=\dots=a_n=1$ in
Theorem~\ref{thm:linearupper} gives a weaker result than
Theorem~\ref{thm:sob} since the resulting sample complexity is $\log(n)
\cdot \tilde{O}(\eps^{-2})$, whereas Theorem~\ref{thm:sob} gives a
$\poly(1/\eps)$ sample complexity bound independent of $n.$

\ignore{

\section{Proof of Theorem~\ref{thm:linearupper}:  Linear transformation functions with $O(1)$
different weights.} \label{sec:O(1)weights}

Recall Theorem~\ref{thm:linearupper}:

\medskip

\noindent {\bf Theorem~\ref{thm:linearupper}} \emph{Let $f(x) = \sum_{i=1}^n a_i x_i$ be any function such that there are at most $k$ different values in the set $\{a_1,\dots,a_n\}.$  Then there is an algorithm that uses $k \log(n) \cdot \widetilde{O}(\eps^{-2})$ samples
from the target distribution $f(\overline{p})$, runs in time $\poly(n^k \cdot \eps^{-k\log^2(1/\eps)})$, and with probability at
least $9/10$ outputs a hypothesis vector $\hat{p}$ such that $\dtv(f(\overline{p}),f(\hat{p})) \leq \eps.$}

\medskip

}

\smallskip

\begin{prevproof}{Theorem}{thm:linearupper}
We claim that the algorithm of Theorem~\ref{thm:log-cover-size} has the desired sample complexity and can be implemented to run
in polynomial time.

Let $\{b_j\}_{j=1}^k$ denote the set of distinct weights and $n_j = \big|i \in [n] \mid a_i =b_j \big|$, where $k = O(1)$. With
this notation, we can write $f(X) = \littlesum_{j=1}^k b_j S_{j} = g(S)$, where $S =(S_1, \ldots, S_k)$ with each $S_j$ a sum
of $n_j$ many independent Bernoulli random variables and $g(y_1, \ldots, y_k) = \sum_{j=1}^k b_jy_j$. Clearly,
$\littlesum_{j=1}^k n_j = n$. By Theorem~\ref{thm: sparse cover theorem}, $S_j$ has an explicit $\eps$-cover
$\mathcal{S}^j_{\eps}$ of size $|\mathcal{S}^j_{\eps}| \le n_j^3 \cdot O(1/\eps) + n \cdot (1/\eps)^{O(\log^2 1/\eps)}$. By
independence across $S_j$'s, the product $\mathcal{Q} = \littleprod_{j=1}^k \mathcal{S}^j_{\eps}$ is an $\eps$-cover for $S$,
hence also for $f(X)$. That is, $f(X)$ has an explicit $\eps$-cover of size $|\mathcal{Q}| = \littleprod_{j=1}^k
|\mathcal{S}^j_{\eps}| \leq (n/k)^{3k} \cdot (1/\eps)^{k \cdot O(\log^2 1/\eps)}$. The sample complexity bound follows directly.

It remains to argue about the time complexity. Note that the running time of the algorithm is $O(|\mathcal{Q}|^2)$ times the
running time of a competition. We will show that a competition between $\bbQ_1, \bbQ_2 \in \mathcal{Q}$ can be efficiently
computed. This amounts to efficiently computing the probabilities $p_1 = f(\bbQ_1)(\mathcal{W}_1)$ and $q_1 =
f(\bbQ_2)(\mathcal{W}_1)$. Note that $\mathcal{W} =f(\{0,1\}^n)= \littlesum_{j=1}^k{b_i} \cdot \{0,1,\ldots,n_j \}$. Clearly,
$|\mathcal{W}| \leq \littleprod_{j=1}^k (n_j+1)
= O((n/k)^k)$. It is thus easy to see that $p_1, q_1$ can be efficiently
computed as long as there is an efficient algorithm for the following
problem: given $\bbQ \in \mathcal{Q}$ and $w \in \mathcal{W}$,
compute $f(\bbQ)(w)$.
Indeed, fix any such $\bbQ, w.$  We have
that $f(\bbQ)(w) = \sum_{m_1,\dots,m_k}
\littleprod_{j=1}^k \Pr_{X \sim \bbQ}[S_j = m_j]$,
where the sum is over all $k$-tuples $(m_1,\dots,m_k)$ such that
$0 \leq m_j \leq n_j$ for all $j$ and $b_1 m_1 + \cdots + b_k m_k = w$
(as noted above there are at most $O((n/k)^k)$ such $k$-tuples).
To complete the proof of Theorem~\ref{thm:linearupper}
we note that $\Pr_{X \sim \bbQ}[S_j = m_j]$ can be computed in $O(n_j^2)$
time by standard dynamic programming.
\end{prevproof}

\section{Conclusion and open problems} \label{sec:conclusion}

We feel that the transformed product distribution learning model offers a rich field for further study, with many natural directions to explore.  We close this paper with some specific questions
 and suggestions for future work.

\medskip

\noindent {\bf Optimally learning sums of Bernoullis?}  An obvious question is whether the competing
advantages of Theorems~\ref{thm:sob} and~\ref{thm:sob2} can be simultaneously achieved by a
single algorithm: is there an algorithm to learn sums of Bernoullis
that uses $\poly(1/\eps)$ samples and runs in $\poly(\log n, 1/\eps)$ time?

\medskip

\noindent {\bf Learning weighted sums of Bernoullis?}   In Section~\ref{sec:constantnumofwts} we observed that a poly$(n)$-time algorithm exists for any linear
transformation function $f(x_1,\dots,x_n) = \sum_{i=1}^n a_i x_i$ in which there are only $O(1)$ many different $a_i$'s.  Can a poly$(n)$-time algorithm be obtained for every linear transformation function
$f(x) = \sum_{i=1}^n a_i x_i$, where $(a_1,\dots,a_n)$ is an arbitrary
vector in $\R^n$?  What if each $a_i$ is a positive integer that is at most poly$(n)$?

\medskip

\noindent {\bf Learning when the transformation function is in $NC^0$?}  Suppose that the
transformation function $f$ maps $\{0,1\}^n$ to $\{0,1\}^n$, i.e. $f = (f_1,\dots,f_n)$,
where each $f_i: \{0,1\}^n \to \{0,1\}$ is a $k$-junta -- a function that depends only on
$k$ of the $n$ input variables -- for some constant $k.$  Is the corresponding transformed product
distribution learning problem solvable in poly$(n,1/\eps)$ time?
We conjecture that the answer is yes.  (Note that as suggested by the second example of
the Introduction,  an algorithm for a special case of this
question (in which each $f_i$ is a particular $(2k-1)$-junta) yields a poly$(n/\eps)$-time
algorithm for learning a mixture of $k$ product distributions over $\{0,1\}^n.$  This mixture learning
problem is indeed known to be solvable in poly$(n/\eps)$ time for constant $k$ but the algorithm is somewhat involved, see~\cite{FOS:08}.)

\medskip

\noindent {\bf When do $O(1)$ samples information-theoretically suffice?}  Our main result in Section~\ref{sec:SOB} shows that for the transformation function $f(x) = \sum_{i=1}^n x_i$, the sample complexity required for learning to accuracy $\eps$ is poly$(1/\epsilon)$ independent of $n$. But as we show in Appendix~\ref{sec:infolower}, the seemingly similar linear transformation function
$f(x) = \sum_{i=n/2+1}^n i x_i$ requires $\Omega(n)$ samples, which is close to the worst possible for any $f$.  This disparity motivates the following question:  what necessary or sufficient conditions can be given on a function $f: \{0,1\}^n \to \R$ that cause the corresponding learning problem to have sample complexity depending only
on $\eps$ (independent of $n$)?  More ambitiously, is there a quantitative measure of the ``complexity'' of a function $f$ that gives a tight quantitative bound on the sample complexity of the product distribution
learning problem for $f$?  The Vapnik-Chervonenkis dimension of a concept class ${\cal C}$ plays
such a role in the PAC learning model, since it tightly characterizes the number of examples
that are required to solve the learning problem for ${\cal C}.$  Is there an analogous measure of the ``complexity'' of a transformation $f$ for our product distribution learning problem?

\bibliography{allrefs}

\appendix

\section{Proof of Observation~\ref{obs:andgate}:  AND-gate functions} \label{ap:andgatefunction}

Let $\overline{p}=(p_1,\dots,p_n)$ be the unknown target vector of probabilities.  For each $i \in [m]$ let $P_i$ denote the true probability $\Pr_{X}[f_i(X) = 1]$ where the probability is
over $X$ drawn from the product distribution $\overline{p}.$  Using poly$(n,1/\eps)$ random samples of $f(X)$
it is straightforward to obtain upper and lower bounds $0 \leq P_{i,-} < P_{i,+} \leq 1$
such that $P_{i,+}-P_{i,-} \leq {\frac \eps {n}}$ and with probability at least $9/10$,
every $i \in [m]$ has $P_{i,-} \leq P_i \leq P_{i,+}.$

For each $i \in [m]$ we have that the function $f_i(x)$ is $\bigwedge_{i \in S_i} x_i$; by independence we have
\[
P_i = \prod_{i \in S_i} p_i \quad
\text{and thus}\quad \log P_i=
\sum_{i \in S_i} \log p_i.
\]

Using the bounds $P_{i,-}$ and $P_{i,+}$ it is straightforward to set up a system
of linear inequalities in variables $q_1,\dots,q_n$ where each $q_i$ plays the role of
$\log p_i$, i.e. for a given $i$ we have the inequalities

\[
\log P_{i,-} \leq \sum_{i \in S_i} q_i \leq \log P_{i,+}.
\]

(We also include the inequalities $q_i \leq 0$ for each $i$ since the $p_i$'s must
be probabilities, i.e. values at most 1.)
With probability at least 9/10 the system is feasible (since setting each $q_i$ to be
$\log p_i$ gives a feasible solution), so we can use polynomial-time linear
programming to obtain a feasible solution $\hat{q}_1,\dots,\hat{q}_n$.  The corresponding
product distribution $\hat{p} = (\hat{p}_1,\dots,\hat{p}_n)$ where $\hat{p}_i = 2^{\hat{q}_i}$
has the property that for each $i$, we have $|\Pr_{X \sim \hat{p}}[f_i(X)=1] - P_i| \leq {\frac \eps
{m}}$.  A simple argument (using independence between the different $f_i(X)$'s,
which holds since the sets $S_i$ are pairwise disjoint) then shows that the total variation distance
$\dtv(f(\hat{p}),f(\overline{p}))$ is at most $\eps$,
and Observation~\ref{obs:andgate} is proved. \qed

\section{Proof of Theorem~\ref{thm:log-cover-size}:  A tournament
between distributions in a cover} \label{ap:tournament}

Recall Theorem~\ref{thm:log-cover-size}:

\medskip

\noindent
{\bf Theorem~\ref{thm:log-cover-size}} \emph{
Fix a function $f: \{0,1\}^n \to \Omega$ where $\Omega$ is any range set. Suppose there exists a $\delta$-cover for $f(\cal{S})$ of size $N = N(n,\delta)$. Then
there is an algorithm that uses $O(\delta^{-2}\log N)$ samples and solves the $f$-transformed product distribution learning problem to accuracy $6\delta$.}

\medskip

\begin{proof}
Let $\bbP \in \mathcal{S}$ be the input distribution fed to the circuit $f$. We will describe an algorithm that, given $m = O(\delta^{-2}\log
N )$ independent samples $s = \{s_i\}_{i=1}^m$ from $f(\bbP)$, finds  a distribution $\bbQ^{\ast} \in \mathcal{S}$ that
satisfies $\dtv (f(\bbP), f(\bbQ^{\ast})) \leq 6\delta$ with probability at least $9/10$.

Recall that the high-level idea of the proof is as follows.  For a pair of distributions $\bbQ_1, \bbQ_2 \in
\mathcal{S}$, we will define a {\em competition between $\bbQ_1$ and $\bbQ_2$} that takes as input the sample $s$ and either crowns one of $\bbQ_1,\bbQ_2$ as the winner of the competition or calls the competition a draw.  Let $\mathcal{Q} \subseteq
\mathcal{S}$ be a $\delta$-cover for $f(\mathcal{S})$ of cardinality $N = N(n, \delta)$. The algorithm performs a tournament
between every pair of distributions from ${\cal Q}$ and outputs a distribution $\bbQ^{\ast} \in {\cal Q}$ that was never a loser
(i.e. won or was a draw in all competitions). If no such distribution exists, the algorithm outputs ``failure.''

To describe the competition procedure between two distributions $\bbQ_1, \bbQ_2 \in {\cal S}$, we define the following
partition of the range space $\mathcal{W} = f(\{0,1\}^n) \subseteq \Omega$:

%\vspace{-0.1cm}

\begin{equation*}
{\cal W}_1:=\{w \in \mathcal{W}~|~f(\bbQ_1)(w) \ge f(\bbQ_2)(w) \};  \quad \quad \quad{\cal W}_2:=\mathcal{W}\setminus\mathcal{W}_1.
\end{equation*}

%\vspace{-0.1cm}

\noindent Let $p_1 = f(\bbQ_1)({\cal W}_1)$ and $q_1 = f(\bbQ_2)({\cal W}_1)$, and define $p_2=1-p_1$ and $q_2=1-q_1$. Clearly,
$p_1 \ge q_1$ and $p_2 < q_2$. Moreover, $\dtv(f(\bbQ_1), f(\bbQ_2)) = p_1-q_1$. Finally, let $T(s)= {1 \over m} | \{i~|~s_i
\in {\cal W}_1 \}|$ be the fraction of samples falling in the set ${\cal W}_1$. The outcome of the competition between $\bbQ_1$
and $\bbQ_2$ is decided as follows: %\vspace{-0.1cm}
\begin{itemize}
\item If $p_1-q_1 \le 5 \delta$, return ``draw'';
%\vspace{-0.1cm}
\item else if $T(s) > p_1 - {3 \over 2} \delta$, return $\bbQ_1$;
%\vspace{-0.1cm}
\item else if $T(s) < q_1 + {3 \over 2} \delta$, return $\bbQ_2$;
%\vspace{-0.1cm}
\item else return ``draw''.
\end{itemize}
%\vspace{-0.1cm}
Observe that the outcome of the competition does not depend on the ordering of the pair of distributions given in the input;
i.e. on inputs $(\bbQ_1,\bbQ_2)$ and $(\bbQ_2,\bbQ_1)$ the competition outputs the same result for a fixed sequence of samples
$s_1,\ldots,s_m$.

We now prove correctness. Our first lemma quantifies the following intuitive fact: If $f(\bbQ_1)$ is ``close''
to $f(\bbP)$ while $f(\bbQ_2)$ is ``far'' from $f(\bbP)$, then with very high probability over the sample
$\bbQ_1$ will be the winner of the competition.

% \vspace{-0.2cm}

\begin{lemma} \label{lem:kostas3}
Let $\bbQ_1 \in {\cal S} $ be such that $\dtv(f(\bbP), f(\bbQ_1)) \le \delta.$

(i) If $\bbQ_2 \in {\cal S}$ is such that $\dtv(f(\bbP), f(\bbQ_2))>6 \delta$, the probability that the competition between
$\bbQ_1$ and $\bbQ_2$ does not return $\bbQ_1$ as the winner is at most $e^{- {m \delta^2 /8 } }$.

(ii)  If $\bbQ_2 \in {\cal S}$ is such that $\dtv(f(\bbP), f(\bbQ_2))>4 \delta$, the probability that the competition between $\bbQ_1$
and $\bbQ_2$ returns $\bbQ_2$ as the winner is at most $e^{- {m \delta^2/ 8 } }$.

\end{lemma}
% \vspace{-0.2cm}
\begin{proof}
Let $r$ denote $f(\bbP)({\cal W}_1)$. The definition of the total variation distance implies that $|r-p_1| \le \delta$. Let us
define the $0/1$ (indicator) random variables $\{Z_i\}_{i=1}^m$, as $Z_i=1$ iff $s_i \in W_1$. Clearly, $T(s)={1 \over m}
\sum_{i=1}^m Z_i$ and $\mathbb{E}[T]=\mathbb{E}[Z_i]=r$. Since the $Z_i$'s are mutually independent, it follows from the
Chernoff bound that $\Pr[T\le r-{\delta/2}] \le e^{- {m \delta^2/8 } }.$ Using $|r-p_1| \le \delta$ we get that $\Pr[T\le
p_1-{3\delta/2}] \le e^{- {m \delta^2/8 } }.$

We prove part (i) first.  Since $\dtv(f(\bbP), f(\bbQ_2)) > 6 \delta$, the triangle inequality implies that $p_1-q_1 =
\dtv(f(\bbQ_1), f(\bbQ_2))> 5 \delta$. Hence, with probability at least $1-e^{- {m \delta^2/8 } }$ the competition between $\bbQ_1$
and $\bbQ_2$ will output $\bbQ_1$ as the winner.

Now we prove part (ii).  If $p_1 - q_1 \leq 5 \delta$ then the competition returns ``draw'' and $\bbQ_2$ is not the winner. If
$p_1 - q_1 > 5 \delta$ then the above implies that the competition between $\bbQ_1$ and $\bbQ_2$ will output $\bbQ_2$ as the winner with
probability at most $e^{- {m \delta^2/8 } }$.
\end{proof}

Our second lemma completes the proof of Theorem~\ref{thm:log-cover-size}.

% \vspace{-0.2cm}

\begin{lemma} \label{lem:6delta}
If $m= \Omega (\delta^{-2}\log N)$, then with probability at least $9/10$ the tournament outputs some distribution
$\bbQ^{\ast} \in {\cal \bbQ}$ such that $\dtv(f(\bbQ^{\ast}), f(\bbP)) \le 6 \delta$.
\end{lemma}

\begin{proof}
Since $\cal Q$ is a $\delta$-cover of $f(\mathcal{S})$, there exists $\widetilde{\bbQ} \in \mathcal{Q}$ such that
$\dtv(f(\widetilde{\bbQ}),f(\bbP)) \leq \delta.$ We first argue that with high probability this distribution $\widetilde{\bbQ}$ never
loses a competition against any other $\bbQ' \in \mathcal{Q}$ (so the tournament does not output ``failure'').  Consider any $\bbQ'
\in {\cal Q}$. If $\dtv(f(\bbP), f(\bbQ')) > 4 \delta$, by Lemma~\ref{lem:kostas3}(ii) the probability that $\widetilde{\bbQ}$ loses to
$\bbQ'$ is at most $e^{-m \delta^2 /8} = O(1/N).$ On the other hand, if $\dtv(f(\bbP), f(\bbQ')) \leq 4 \delta$, the triangle
inequality gives that $\dtv(f(\widetilde{\bbQ}), f(\bbQ')) \leq 5 \delta$ and thus $\widetilde{\bbQ}$ draws against $\bbQ'.$  A union
bound over all $N$ distributions in ${\cal Q}$ shows that with probability $19/20$, the distribution $\widetilde{\bbQ}$ never loses
a competition.

We next argue that with probability at least $19/20$, every distribution $\bbQ' \in \mathcal{Q}$ that never loses has $f(\bbQ')$
close to $f(\bbP).$ Fix a distribution $\bbQ'$ such that $\dtv(f(\bbQ'), f(\bbP)) > 6 \delta$; Lemma~\ref{lem:kostas3}(i) implies that
$\bbQ'$ loses to $\widetilde{\bbQ}$ with probability $1 - e^{-m \delta^2 /8} = 1 - O(1/N)$.  A union bound gives that with
probability $19/20$, every distribution $\bbQ'$ that has $\dtv(f(\bbQ'), f(\bbP)) > 6 \delta$ loses some competition.

Thus, with overall probability at least $9/10$, the tournament does not output ``failure'' and outputs some distribution
$\bbQ^{\ast}$ such that $\dtv(f(\bbP), f(\bbQ^{\ast}))$ is at most $6 \delta.$
This proves Lemma~\ref{lem:6delta} and Theorem~\ref{thm:log-cover-size}.
\end{proof}
\end{proof}

\section{Learning transformed product distributions can be computationally hard} \label{ap:nphard}

Recall Theorem~\ref{thm:nphard}:

\medskip

\noindent {\bf Theorem~\ref{thm:nphard}}  \emph{
Suppose NP $\not \subseteq$ BPP.  Then there is an explicit
degree-2 polynomial $f$ (given in Equation~(\ref{eq:deg2}) below)
such that there is no polynomial-time algorithm
that solves the transformed product distribution learning problem for $f$ to accuracy
$\eps = 1/3.$}

\medskip

\begin{proof}
The function $f$ is quite simple.  It takes $m = n^2 + n$ bits of input
\[
(w,s) = (w_{1,1},\dots,w_{1,n},w_{2,1},\dots,w_{2,n},\dots,w_{n,1},\dots,
w_{n,n},s_1,\dots,s_n).
\]

\noindent Here we think of each $n$-bit substring $w_{i,1} \dots w_{i,n} \in \{0,1\}^n$ as the binary representation of the
number $W_i = \littlesum_{j=1}^{n} 2^{n-j}w_{i,j} \in \{0,1,\dots,2^{n}-1\}$. We think of $s_1,\dots,s_n$ as representing a
subset $S \subseteq [n]$, where $i \in S$ if and only if $s_i=1.$ The function $f(w,s)$ is defined to be

\begin{equation}
f(w,s) = \littlesum_{i=1}^n 2^{2in} W_{n+1-i} + \littlesum_{i=1}^n W_i (2s_i - 1) \label{eq:deg2}
\end{equation}
which is easily seen to be a degree-2 polynomial.

Recall that an input to the NP-complete PARTITION problem
is a list of $n$ numbers $W_1,\dots,W_n$.  The input is a yes-instance if there
is a set $S \subseteq [n]$ such that $\sum_{i \in S} W_i = \sum_{i \notin S} W_i$ (
equivalently, if there is a bitstring $(s_1,\dots,s_n)$ such that
$\sum_{i=1}^n W_i (2 s_i-1) = 0$).

It is easy to see from the definition of $f$ that the output number $f(w,s)$
can be viewed (reading from most significant bit to least significant bit) as specifying

\begin{itemize}
\item the $n$ numbers $W_1,\dots,W_n$; and
\item the value $\sum_{i=1}^n W_i (2 s_i-1)$ of the candidate solution $S \subseteq [n]$.  Note that this value is 0
if and only if $S$ is a legitimate solution to the PARTITION instance
specified by $(W_1,\dots,W_n)$.
\end{itemize}

So we may view the input to $f$ as the tuple $(W_1,\dots,W_n,S)$
and its corresponding output as the tuple $(W_1,\dots,W_n,v)$
where $v = \sum_{i=1}^n W_i (2 s_i-1) \in {\mathbb N}$ is the value of the candidate solution $S$.

%\bigskip \bigskip
%
%
%
%as described below
%and outputs Intuitively, it takes
%It takes as input $(\varphi,x) \in \{0,1\}^{m+n}$
%where $\varphi$ is an $m$-bit string that encodes a 3-CNF formula over $n$
%variables and $x$ is an assignment in $\{0,1\}^n$.
%Given this input, $C$ outputs the pair $(\varphi, \varphi(x))$.

The proof is by contradiction, so let us suppose that $L$ is a
learning algorithm that runs in polynomial time
and learns to accuracy $\eps = 1/3.$
For any ``target distribution'' $\overline{p} \in [0,1]^{m}$,
if $L$ is given access to
$q(n)=\poly(n)$ many independent draws from $f(\overline{p})$, then
with probability $9/10$ algorithm $L$ outputs a
vector $\hat{p} = (\hat{p}_1,\dots,\hat{p}_{m})$ of probabilities s.t.
the total variation distance $\dtv(f(\overline{p}),
f(\hat{p}))$ is at most $1/3.$

We now explain how $L$ yields a randomized poly$(n)$-time algorithm $A$
to solve PARTITION.  Given a PARTITION instance $(W_1,\dots,W_n)$ as input,
algorithm $A$ runs $L$ on a data set consisting of $q(n)$ copies of the tuple
$(W_1,\dots,W_n,0).$  If $L$ fails to return a vector $\hat{p}$ then
$A$ outputs ``no'' (meaning that the PARTITION instance has no solution).
If $L$ returns a vector $\hat{p}$
then:
\begin{itemize}

\item $A$ checks that $\prod_{i=1}^{n^2} \max\{\hat{p}_i, 1-\hat{p}_i\} \geq 2/3$;
if not it returns ``no.''

\item If $A$ reaches this step,  for each
$i \in [n^2]$ let $b_i$ be the result of rounding $\hat{p}_i$
to the nearest integer (0 or 1) and let $(W'_1,\dots,W'_n)$ be the PARTITION instance
represented by the string $(b_1,\dots,b_{n^2})$.
$A$ checks that $(W'_1,\dots,W'_n)$ is identical to $(W_1,\dots,W_n)$; if not
it outputs ``no.''

\item If $A$ reaches this step, then $A$ draws 100 random
$n$-bit strings $x^1,\dots,x^{100}$ independently from the
product distribution $(\hat{p}_{n^2+1},\dots,\hat{p}_{n^2+n})$
and evaluates $\sum_{i=1}^n W'_i(2x_i - 1)$ on each of them.
If any evaluation yields 0 then $A$ outputs ``yes'' and otherwise
it evaluates ``no.''

\end{itemize}

\medskip

We now prove correctness.  Suppose first that $(W_1,\dots,W_n)$ is a yes-instance of PARTITION.  With probability 1 the data
set consisting of $q(n)$ copies of $(W_1,\dots,W_n,0)$ is identical to the outcome of $q(n)$ draws from the distribution
$f(p^\star)$, where each coordinate of $p^\star$ is either 0 or 1, the first $n^2$ coordinates $p^\star_1,\dots,p^\star_{n^2}$
encode the numbers $(W_1,\dots,W_n)$, and the last $n$ coordinates encode a legitimate solution $S$ for $(W_1,\dots,W_n).$
Thus, with probability at least $9/10$ the algorithm $L$ outputs a vector $\hat{p} = (\hat{p}_1,\dots,\hat{p}_{m})$ s.t.
$\dtv(f(\hat{p}),f(p^\star)) \leq 1/3.$ Since $f(p^\star)$ puts all its weight on one output string $(W_1,\dots,W_n,0)$, it
must indeed be the case that $\prod_{i=1}^{n^2} \max\{\hat{p}_i, 1-\hat{p}_i\} \geq 2/3$, and it is easy to see that the string
$b=(b_1,\dots,b_{n^2})$ defined in the second bullet will be identical to $(W_1,\dots,W_n).$ Thus $(W'_1,\dots,W'_n)$ is
identical to $(W_1,\dots,W_n)$, and $A$ makes it through the second bullet. Finally, since $\dtv(f(\hat{p}),f(p^\star)) \leq
1/3$, in expectation at least $2/3$ of the $100$ strings independently drawn from $(\hat{p}_{n^2+1},\dots,\hat{p}_{m})$ should
be legitimate solutions, and the probablity that none of $x^1,\dots,x^{100}$ is a legitimate solution is at most $0.001.$ Thus,
if $(W_1,\dots,W_n)$ is a yes-instance of PARTITION, the overall probability that $A$ outputs ``yes'' is at least 0.89.

Now suppose that $(W_1,\dots,W_n)$ is a no-instance of PARTITION, but that
$A$ outputs ``yes'' on $(W_1,\dots,W_n)$ with probability at least
$0.1.$  This means that with probability at least $0.1$,
$L$ outputs a vector $\hat{p}$ such that
$\prod_{i=1}^{n^2} \max\{\hat{p}_i, 1-\hat{p}_i\} \geq 2/3$, which
can be uniquely decoded into a PARTITION instance $(W'_1,\dots,W'_n)$ which
must equal $(W_1,\dots,W_n).$  The PARTITION instance $(W'_1,\dots,W'_n) = (W_1,\dots,W_n)$
must be satisfied with probability
at least $1/1000$ by a random string drawn
from the probability distribution $(\hat{p}_{n+1},\dots,
\hat{p}_{m})$ (for otherwise the probability of a ``yes''
output would be less than $1/10$).  But this violates the assumption
that $(W_1,\dots,W_n)$ is a no-instance.  So if $(W_1,\dots,W_n)$ is a no-instance
of PARTITION, then it must be the case that $A$ outputs ``no''
on $(W_1,\dots,W_n)$ with probability less than $0.1$.

Thus we have shown that $A$ is a BPP algorithm for the PARTITION problem,
and Theorem~\ref{thm:nphard} is proved.
\end{proof}

\section{Proof of Theorem~\ref{thm:linearlower}:  $f(x) = \sum_{i=k/2+1}^k i x_i$
requires $\Omega(k)$ samples} \label{sec:infolower}

We define a probability distribution over problem instances (i.e. target probability vectors $\overline{p}$) as follows: A subset $S \subset \{k/2 + 1, \dots, k\}$ of size $|S|=k/100$ is
drawn uniformly at random, i.e. each of the ${k/2 \choose k/100}$ outcomes
for $S$ is equally likely.
For each $i \in S$ the value $p_i$ equals $100/k = 1/|S|,$ and for
all other $i$ the value $p_i$ equals 0.  We will need two easy lemmas:

%We note some useful properties of this construction.

\begin{lemma} \label{lemma:l1}
Fix any $S,\overline{p}$ as described above.
%Let $X$ be the corresponding sum of weighted Bernoullis
%$X = \sum_{i=1}^n a_i X_i$ with $\E[X_i]=p_i.$
%
For any $j \in \{k/2 + 1,\dots,k\}$ we have $f(\overline{p})(j) \neq 0$ if and
only if $j \in S$.  For any $j\in S$ the value  $f(\overline{p})(j)$ is exactly
$(100/k)(1 - 100/k)^{k/100 - 1} > 35/k$ (for $k$ sufficiently large),
and similarly $f(\overline{p})(\{k/2+1,\dots,k\})>0.35$ (again for $k$ sufficiently large).

\end{lemma}
The first claim of the lemma holds because any set of $c \geq 2$ numbers from $\{k/2+1,\dots,k\}$
must sum to more than $k$.
The second claim holds because the only way a draw of $X$ from $\overline{p}$ can have $f(X)=j$ is if $X_j=1$ and all other $X_i$ are 0 (and uses $\lim_{x \leftarrow 0^+}(1 - 1/x)^x = 1/e$).

The next lemma is an easy consequence of Chernoff bounds:

\begin{lemma} \label{lemma:l2}
Fix any $\overline{p}$ as defined above, and consider a sequence of $k/2000$ independent draws of $X$
from $\overline{p}.$
With probability $1-e^{-\Omega(k)}$ the total number of indices
$j \in [k]$ such that $X_j=1$ in any of the $k/2000$ draws is at
most $k/1000$.
\end{lemma}
%This is an easy application of Chernoff bounds, using the fact that
%the expected number of instances of having any $X_j$ equals 1 in any draw
%is precisely $w/2000$ (since each draw of $X$ is expected to have exactly
%one $X_j$ equalling 1).

%
%\subsection{The lower bound}
%

\noindent {\bf Proof of Theorem~\ref{thm:linearlower}:}
Let $L$ be a learning algorithm that receives $k/2000$ samples.  Let $S \subset
\{k/2+1,\dots,k\}$ and the target distribution $\overline{p}$ be chosen randomly
as defined above.

We consider an augmented learner $L'$ that is given ``extra information.''
For each point $f(X)$ in the sample, instead of receiving $f(X)=\sum_{i=k/2+1}^k i X_i$,
the learner $L'$ is given the entire vector
$(X_1,\dots,X_k) \in \{0,1\}^k$.  Let $T$ denote the set of elements $j \in \{k/2+1,\dots,k\}$ for which the learner is given some vector $X$ that has $X_j=1.$  By Lemma~\ref{lemma:l2} we have
$|T| \leq k/1000$ with probability at least $1 - e^{-\Omega(k)}$; we condition on the event
$|T| \leq k/1000$ going forth.

Fix any value $\ell \leq k/1000.$  Conditioned on $|T|=\ell,$
the set $T$ is equally likely to be any $\ell$-element subset of $S$,
and all possible ``completions'' of $T$ with
an additional $k/100-\ell \geq 9k/1000$ elements of $\{k/2+1,\dots,k\} \setminus T$
are equally likely to be the true set $S$.

\ignore{

 Since the learner is guaranteed that each
$p_i$ is either 0 or $100/w$, ``frequency'' information of how often
a given value $j$ has $X_j=1$ among the $w/2000$ draws is of no use;
seeing $X_j=1$ even a single time tells the learner that $p_j$ is
exactly $100/w$.\footnote{I know this is informal; we could make it formal
by arguing that given only the information about the set of elements that
are seen without ``frequency information'', it is possible to simulate
the correct distribution of frequencies that would be seen in the sample.}
So all the information the learner has about $\overline{p}$ is the identity of some
set of $\ell \leq w/1000$ elements that the learner knows must belong to $S$.
Again we give the learner extra information by telling her an additional (randomly chosen)
$w/1000-\ell$ elements of $S$.}

Let $\hat{p}=(\hat{p}_1,\dots,\hat{p}_k)$ denote the hypothesis vector of probabilities
that $L'$ outputs.
\ignore{ and let $\hat{X}$ be the corresponding distribution
of $\sum_{i=1}^k a_i \hat{X}_i.$\footnote{Of course, if the learner has any brains she will
set $\hat{p}_i=100/w$ for each $i \in T$, but we
do not assume this.}
}  Let $R$ denote the set $\{k/2 + 1,\dots,k\} \setminus T$; note that since
$|T|=\ell \leq k/1000$, we have $|R| \geq 499k/1000.$
We consider two possible cases for $\hat{p}$ and show that in either case the learner's hypothesis distribution $f(\hat{p})$
has high error (in the first case
because of outcomes in $R$, in the second case because of the
outcome 0) with high probability.

%\begin{itemize}

%\item

 {\bf Case 1:  Fewer than $k/4$ of the (at least $499k/1000$) elements ${i \in R}$
have $\hat{p}_i \geq 10/k.$}  Let $U$ be the set of those (fewer than $k/4$) elements.
Since every $(k/100 - \ell)$-element subset of $R$ is equally likely
to be the correct completion $S \setminus T$, and as observed above $k/100-\ell$ is at least $9k/1000$, an elementary argument shows that with probability $1-e^{-\Omega(k)}$ we have that $U$
contains at most $8k/1000$ of the $k/100-\ell \geq 9k/1000$ elements of $S \setminus T.$  Assuming this happens, Lemma~\ref{lemma:l1} now implies that each of the (at least) $k/1000$
``missed'' elements (that are in $S \setminus T$ but not in $U$) contributes at least $35/k -  10/k > 25/k$ to the total variation distance between $f(\overline{p})$ and $f(\hat{p}).$  (This is because each point in $(S \setminus T) \setminus U$ has probability \ignore{at least $35/k$ under $f(\overline{p})$ and probability} at most $10/k$ under $f(\hat{p})$; the only way to get such an outcome $i$ from $f(\hat{p})$ is for the draw of $\hat{X}$ from $\hat{p}$ to have $\hat{X}_i=1$, which occurs with probability
$10/k.$)  So in this case  $\dtv(f(\overline{p},f(\hat{p}))$ is at least $25/1000=1/40.$

%\item

{\bf Case 2: At least $k/4$ of the (at least $499k/1000$) elements ${i \in R}$
have $\hat{p}_i \geq 10/k.$}
In this case, we have $f(\hat{p})(0)
\leq (1-10/k)^{k/4} < 1/10.$
Since the target distribution $f(\overline{p})$ has $f(\overline{p})(0) = (1-100/k)^{k/100} > 0.35,$ it follows that
$\dtv(f(\overline{p}),f(\hat{p})) \geq 1/4$ in this case. \qed

%\end{itemize}

%\newpage

\section{A Simple Proof of the DKW Inequality} \label{ap:DKW}

Recall the framework of the DKW inequality:  $X$ is any random variable supported on $\{0,1,\dots,n\}$.
$Z_1,\ldots,Z_k$ are independent copies of $X$, and $Z_i^{(\ell)}$ is defined to be $\ind_{Z_i \le \ell}$,
for all $\ell=0,\ldots,n$ and $i=1,\ldots,k$.  Finally, define $\hat{F}_X(\ell) := {\sum_iZ_i^{(\ell)} / k}$.

We prove:

\medskip

\noindent {\bf Theorem~\ref{thm:cumulative distribution approximation}
(DKW Inequality)}
\emph{ Let $k=\max\{576,(9/8)\ln(1/\delta)\}\cdot
(1/\eps^2).$ Then with probability at least $1 - \delta$ we have
$$\max_{0 \le \ell \le n}\left|\hat{F}_X(\ell) - {F}_X(\ell)\right| \le \epsilon.$$
}

We first prove the following weaker version of the inequality:

\begin{theorem} \label{thm:weakdkw}
Let $k={\frac 4{\delta \eps^2}}.$  Then with probability at least $1 - \delta$ we have
$$\max_{0 \le \ell \le n}\left|\hat{F}_X(\ell) - {F}_X(\ell)\right| \le \epsilon.$$
\end{theorem}

\begin{prevproof}{Theorem}{thm:weakdkw}
Let $\L$ be the smallest index in $\{0,\ldots,n\}$ such that ${1\over 2} < F_X(\L)$. We first show the following.

\begin{lemma}\label{lem:one side}
Let $k = {1 \over \epsilon^2 c}$, where $c>0$. Then
$$\Pr\left[ \max_{0\le \ell \le \L-1} \left|\hat{F}_X(\ell) - F_X(\ell)\right| > \epsilon \right] \le 2c.$$
\end{lemma}
\begin{prevproof}{Lemma}{lem:one side}
If $\L=0$ the statement is vacuously true. If $\L >0$, let us denote $\Lambda_{\ell}:= \hat{F}_X(\ell)- F_X(\ell)$, for $\ell = 0,\ldots,\L-1$, and define
\begin{align*}
&\Omega_{-1} := 0\\
&\Omega_{\ell} :=
\begin{cases}
{\Lambda_{\ell} \over 1-F_X(\ell)},~~~~~~~~~~\text{if } {\Lambda_j } \le {\epsilon} \text{ for all } {0\le j \le \ell-1}\\
\text{ }\\
%{1- F_X(i-1) \over 1- F_X(i)} \Omega_{i-1},~~~~~~~\text{otherwise}
\Omega_{\ell-1},~~~~~~~~~~~~~~\text{otherwise}
\end{cases},
\text{~~~~~for $\ell = 0,\ldots, \L-1$.}
\end{align*}
Also, define $\Omega^*_{-1}=0$ and $\Omega^*_\ell := {\Lambda_{\ell} \over 1-F_X(\ell)}$, for $\ell=0,\ldots,\L-1$.
\begin{claim} Both $\{\Omega^*_\ell\}_{\ell=-1,\ldots,\L-1}$ and $\{\Omega_\ell\}_{\ell=-1,\ldots,\L-1}$ are martingale sequences.
\end{claim}
\begin{proof}
Clearly, $\EE{\Omega^*_{0}} = 0 = \EE{\Omega_{0}}$. Moreover, for $\ell=1,\ldots,\L-1$:
\begin{align*}
\EE{ \Lambda_\ell ~\vline~\Lambda_{\ell-1} ={m \over k} -F_X(\ell-1) } &= {m \over k} + {k-m \over k} \Pr[Z_1 = \ell~|~Z_1 \ge \ell] - F_X(\ell)\\
&= {m \over k} + {k-m \over k} {F_X(\ell)-F_X(\ell-1)\over 1-F_X(\ell-1)} - F_X(\ell)\\
&= {1-F_X(\ell) \over 1-F_X(\ell-1)} \cdot \left[ {m \over k} - F_X(\ell-1)\right]\\
&= {1-F_X(\ell) \over 1-F_X(\ell-1)} \cdot \Lambda_{\ell-1}.
\end{align*}
It follows that the sequence $\{\Omega^*_\ell\}_{\ell=-1,\ldots,\L-1}$ is a martingale sequence. Hence, $\{\Omega_\ell\}_{\ell=-1,\ldots,\L-1}$ is also a martingale sequence.
\end{proof}

Note that
\begin{align*}
\sum_{\ell=0}^{\L-1}\EE{(\Omega_\ell - \Omega_{\ell-1})^2} &= \sum_{\ell=0}^{\L-1}\EE{(\Omega_\ell^2 + \Omega_{\ell-1}^2 -2 \Omega_{\ell} \Omega_{\ell-1}}\\
&= \sum_{\ell=0}^{\L-1}\EE{(\Omega_\ell^2 + \Omega_{\ell-1}^2 -2 (\Omega_{\ell}-\Omega_{\ell-1}+\Omega_{\ell-1}) \Omega_{\ell-1}}\\
&= \sum_{\ell=0}^{\L-1}\EE{(\Omega_\ell^2 - \Omega_{\ell-1}^2 -2 (\Omega_{\ell}-\Omega_{\ell-1}) \Omega_{\ell-1}}\\
&= \EE{\Omega_{\L-1}^2} - \EE{\Omega_{-1}^2} -\sum_{\ell=0}^{\L-1} 2 \EE{(\Omega_{\ell}-\Omega_{\ell-1}) \Omega_{\ell-1}}\\
&= \EE{\Omega_{\L-1}^2} - \EE{\Omega_{-1}^2} \\
& = \EE{\Omega_{\L-1}^2},
\end{align*}
where in the second to last line we used the martingale property and in the last line we used that $\Omega_{-1} =0$. The same analysis holds true for the martingale sequence $\Omega^*_{\ell}$ giving:
\begin{align*}
\sum_{\ell=0}^{\L-1}\EE{(\Omega^*_\ell - \Omega^*_{\ell-1})^2}  = \EE{(\Omega^*_{\L-1})^2}.
\end{align*}
In particular, the above imply that $\EE{(\Omega_{\L-1})^2} \le \EE{(\Omega^*_{\L-1})^2}$.
Now we have:
\begin{align*}
\Pr[\max_{\ell} {\Lambda_{\ell} } > {\epsilon}] &\le\Pr\left[{\Omega_{\L-1} } > {\epsilon \over 1-F_X(0)}\right]\\
&\le \Pr\left[{\Omega_{\L-1} } > {\epsilon}\right]\\
&\le {1 \over \epsilon^2} \EE{\Omega_{\L-1}^2}\\
&\le {1 \over \epsilon^2} \EE{(\Omega^*_{\L-1})^2}\\
&= {1 \over \epsilon^2} \cdot {1 \over (1-F_X(\L-1))^2} \cdot \EE{(\Lambda_{\L-1})^2}\\
&= {1 \over \epsilon^2} \cdot {1 \over (1-F_X(\L-1))^2} \cdot \Var{[\Lambda_{\L-1}]}\\
&\le {4 \over \epsilon^2} \cdot \Var{[\Lambda_{\L-1}]}\\
&\le {4 \over \epsilon^2} \cdot \Var{\left[\sum_iZ_i^{(\L-1)} \over k\right]}\\
&= {4 \over \epsilon^2} \cdot {1 \over k} F_X(\L-1) \cdot (1-F_X(\L-1)) \le {1 \over \epsilon^2} \cdot {1 \over k}
\end{align*}
Considering the random variables $\Lambda'_{\ell}:=-\Lambda_{\ell}$ and repeating the analysis above,
we can establish in a similar fashion that
$$\Pr[\min_{\ell} {\Lambda_{\ell} } <- {\epsilon}] \le {1 \over \epsilon^2} \cdot {1 \over k}.$$
Combining the above together with a union bound completes the proof of the lemma.
\end{prevproof}

Now let us denote by $\L'$ the smallest index in $\{0,\ldots,n\}$ such that $1/2<1-F_X(n-\L'-1)$. Applying Lemma~\ref{lem:one
side} to the random variable $n-X$ we obtain the following.
\begin{lemma}\label{lem:the other side}
Let $k = {1 \over \epsilon^2c} $, where $c>0$. Then
$$\Pr\left[ \max_{n-\L' \le \ell \le n-1} \left|\hat{F}_X(\ell) - F_X(\ell)\right| > \epsilon \right] \le 2c.$$
\end{lemma}
Recall that $F_X(\L) > 1/2 > F_X(n-\L'-1)$. Hence, $\L >n- \L'-1$, which implies that $\L-1 \ge n- \L'-1$.

\medskip

This observation, a union bound and Lemmas~\ref{lem:one side} and \ref{lem:the other side} (applied for $c=\delta/4$) imply
Theorem~\ref{thm:weakdkw}.
\end{prevproof}

We are now ready to prove the actual DKW inequality. Let $M$ denote the random variable
$$\max_{0\leq \ell \leq n} |\Lambda_\ell| = \max_{0\leq \ell \leq n}\left|\hat{F}_X(\ell) - {F}_X(\ell)\right|.$$

Since $M$ is defined by the outcomes of $Z_1,\dots,Z_k$, we can write $M = g(Z_1, \ldots, Z_k)$, for some function $g:\R^k \to
\R$. It is also easy to see that $g(Z_1,\dots,Z_k)$ is $1/k$-Lipschitz as a function of its arguments $Z_1,\dots,Z_k$. Since
the $Z_i$s are independent, we can apply McDiarmid's inequality to obtain
$$\Pr\left[M - \E[M] > \eps\right] \leq \exp(- 2 \eps^2 k).$$

By repeated applications of Theorem~\ref{thm:weakdkw} we show the following:

\begin{claim} \label{claim:expectation}
Let $k \geq {\frac {256} {\eps^2}}.$ Then
\[
\E_{Z_1, \ldots, Z_k} \left[ M \right] \le {\frac \eps 2}.
\]
\end{claim}

\begin{prevproof}{Claim}{claim:expectation}
Note that $M$ is supported in $[0,1].$ We have
\begin{eqnarray*}
\E[M] &\leq& {\frac \eps 4}\cdot\Pr[0\le M \le \eps/4] +
\sum_{i=1}^{\lceil \log(4/\eps) \rceil} \left({\frac {2^i \eps}4}\right) \cdot \Pr[{\frac {2^{i-1}\eps} 4} < M \le {\frac {2^{i}\eps} 4} ]\\
&\leq& {\frac \eps 4}+\sum_{i=1}^{\lceil\log(4/\eps)\rceil} \left({\frac {2^i \eps}4}\right) \cdot \Pr[M > {\frac {2^{i-1}\eps} 4} ]\\
&\leq& {\frac \eps 4}+\sum_{i=1}^{\lceil\log(4/\eps)\rceil} \left({\frac {2^i \eps}4}\right) \cdot 2^{-2i}\\
&\leq& {\frac \eps 4} + \sum_{i=1}^\infty {\frac \eps {4 \cdot 2^i}} = {\frac \eps 2}
\end{eqnarray*}
where we used Theorem~\ref{thm:weakdkw} for the third inequality.
\end{prevproof}
Note that Claim~\ref{claim:expectation} imposes a lower bound on the sample complexity. From this claim and McDiarmid's
inequality we get
$$ \Pr\left[M > 3\eps/2\right] \le \Pr\left[M - \E[M] > \eps\right] \leq \exp(- 2 \eps^2 k).$$
Theorem~\ref{thm:cumulative distribution approximation} immediately follows by setting $2\eps/3$ in place of $\eps$. \qed

\section{Proof of Theorem~\ref{thm:sob2}:  Learning sums of Bernoullis in $\poly(\log(n),1/\eps)$ time with efficient hypotheses}

\label{ap:birge}

In this section we observe that the probability theory literature
already provides a non-proper algorithm (\cite{Birge:97}) that
can be used as an alternative to our Theorem~\ref{thm:sob}
for learning a sum of $n$ Bernoulli random variables.  This algorithm
has faster running time, requiring only $\poly(\log n, 1/\eps)$ time steps,
but significantly worse sample complexity of $\log(n) \cdot \poly(1/\eps)$
samples (recall
that Theorem~\ref{thm:sob} requires only $\poly(1/\eps)$ samples independent of $n$).

\smallskip

We say that a distribution $p$ over domain $\{0,1,\dots,n\}$ is \emph{unimodal} if there exists some \emph{mode} $M \in \{0,1,\dots,n\}$
(not necessarily unique) such that the probability density function (pdf) of $p$ is monotone nonincreasing on $\{0,\dots,M\}$ and
monotone nondecreasing on $\{M,\dots,n\}.$  Our starting point is the observation that
any sum of $n$ Bernoullis is a unimodal distribution over $\{0,1,\dots,n\}$.
(This can be shown by a straightforward induction on $n$, see e.g.
\cite{KeilsonGerber:71}.)
This observation lets us use a powerful algorithm due to
\cite{Birge:97} that can
learn any unimodal distribution to accuracy $\eps$ in total variation distance using sample complexity is optimal up to constant factors (\cite{Birge:87}).
(Birg\'e's work deals with unimodal distributions over the
continuous interval $[0,n]$, but it is easily modified to apply to our discrete setting.)
His algorithm uses $O(\log n/\eps^3)$ samples and has overall running time of $O(\log^2 n)/\eps^3$ bit operations (note that this running time is best possible, given the sample complexity,
since each sample is an $\Omega(\log n)$ bit string.).  It outputs a hypothesis distribution that is a histogram over
$O((\log n)/\eps)$ intervals that cover $\{0,\dots,n\}$:  more precisely, the hypothesis is uniform within each interval, and for each interval the total mass it assigns to the interval is simply the fraction of
samples that landed in that interval.  Thus the hypothesis distribution has a succinct description and can be efficiently evaluated
using a small amount of randomness as claimed in Theorem~\ref{thm:sob2}.

\end{document}